\documentclass[11pt]{article} 

\usepackage[letterpaper,margin=1in]{geometry}

\usepackage[letterpaper,margin=1in]{geometry}

\def\colorful{1}
\ifnum\colorful=1

\else

\fi

\newif\ifred
\redtrue

\usepackage[T1]{fontenc}
\usepackage{microtype}

\usepackage{titlesec}
\titleformat*{\paragraph}{\bfseries}

\usepackage{amsthm,amsmath,amssymb,amsfonts,amssymb,mathtools}
\usepackage{amsmath,amssymb,amsfonts,amssymb,mathtools}
\usepackage{xcolor}
\usepackage{empheq}
\usepackage{dsfont}
\usepackage{mathrsfs}
\usepackage{graphicx}
\usepackage{enumitem}
\usepackage{aliascnt} 
\usepackage{xspace}
\usepackage{bbm}

\usepackage{caption}
\usepackage{tikz}
\usetikzlibrary{patterns}
\usetikzlibrary{patterns}
\usetikzlibrary{arrows,shapes,automata,backgrounds,petri,positioning}
\usetikzlibrary{shadows}
\usetikzlibrary{calc}
\usetikzlibrary{spy}
\usetikzlibrary{angles, quotes}
\usetikzlibrary{matrix}
\usepackage{pgf,pgfplots}
\usepackage{pgfmath,pgffor}
\pgfplotsset{compat=1.17}

\usepackage{framed}
 \usepackage{algorithm}
\usepackage[noend]{algpseudocode}

\definecolor[named]{ACMBlue}{cmyk}{1,0.1,0,0.1}
\definecolor[named]{ACMYellow}{cmyk}{0,0.16,1,0}
\definecolor[named]{ACMOrange}{cmyk}{0,0.42,1,0.01}
\definecolor[named]{ACMRed}{cmyk}{0,0.90,0.86,0}
\definecolor[named]{ACMLightBlue}{cmyk}{0.49,0.01,0,0}
\definecolor[named]{ACMGreen}{cmyk}{0.20,0,1,0.19}
\definecolor[named]{ACMPurple}{cmyk}{0.55,1,0,0.15}
\definecolor[named]{ACMDarkBlue}{cmyk}{1,0.58,0,0.21}

\usepackage[colorlinks,citecolor=blue,linkcolor=magenta,bookmarks=true]{hyperref}
 \usepackage[nameinlink]{cleveref}
\creflabelformat{ineq}{#2{\upshape(#1)}#3}
\crefname{sub}{Subsection}{Subsection}
\creflabelformat{Subsection}{#2{\upshape(#1)}#3}
\crefname{sdp}{SDP}{SDP}
\creflabelformat{sdp}{#2{\upshape(#1)}#3}
\crefname{lp}{LP}{LP}
\creflabelformat{lp}{#2{\upshape(#1)}#3}
\usepackage{aliascnt}
\usepackage{cleveref}
\crefname{ineq}{Inequality}{Inequality}
\creflabelformat{ineq}{#2{\upshape(#1)}#3}
\crefname{sub}{Subsection}{Subsection}
\creflabelformat{Subsection}{#2{\upshape(#1)}#3}
\crefname{sdp}{SDP}{SDP}
\creflabelformat{sdp}{#2{\upshape(#1)}#3}
\crefname{lp}{LP}{LP}
\creflabelformat{lp}{#2{\upshape(#1)}#3}



\newtheorem{theorem}{Theorem}[section]

\newtheorem{lemma}[theorem]{Lemma}

\newtheorem{informal theorem}[theorem]{Theorem (informal statement)}

\newtheorem{proposition}[theorem]{Proposition}
\newtheorem{corollary}[theorem]{Corollary}
 
\newtheorem{fact}[theorem]{Fact}

\newtheorem{remark}[theorem]{Remark}

\newtheorem{definition}[theorem]{Definition}

\newcommand{\eqdef}{\coloneqq}

\newcommand{\lp}{\left}
\newcommand{\rp}{\right}
\newcommand\norm[1]{\left\| #1 \right\|}

\renewcommand\vec[1]{\mathbf{#1}}
\DeclareMathOperator*{\pr}{\mathbf{Pr}}
\DeclareMathOperator*{\E}{\mathbf{E}}

\newcommand{\normal}{\mathcal{N}}

\DeclareMathOperator*{\argmax}{argmax}

\newcommand{\bx}{\mathbf{x}}

\newcommand{\err}{\mathrm{err}}

\newcommand{\R}{\mathbb{R}}

\newcommand{\Z}{\mathbb{Z}}
\newcommand{\N}{\mathbb{N}}
\newcommand{\eps}{\epsilon}

\newcommand{\poly}{\mathrm{poly}}

\newcommand{\D}{\mathcal{D}}

\newcommand{\Ind}{\mathds{1}}

\newcommand{\littlesum}{\mathop{\textstyle \sum}}

\newcommand{\x}{\vec x}

\newcommand{\opt}{\mathrm{opt}}

\newcommand{\iid}{{i.i.d.}\ }
\newcommand{\abs}[1]{\lp| #1 \rp|}
\newcommand{\A}{\mathcal{A}}

\renewcommand\Pr{\pr}

\begin{document}

\title{Statistical Query Hardness of Multiclass Linear Classification\\ 
with Random Classification Noise}

\author{
Ilias Diakonikolas\thanks{Supported by NSF Medium Award CCF-2107079 and an H.I. Romnes Faculty Fellowship.}\\
University of Wisconsin-Madison\\
\texttt{ilias@cs.wisc.edu}
\and
Mingchen Ma\thanks{Supported by NSF Award  CCF-2144298 (CAREER).}\\
University of Wisconsin-Madison\\
\texttt{mingchen@cs.wisc.edu}
\and
Lisheng Ren\thanks{Supported in part by NSF Medium Award CCF-2107079.}\\
University of Wisconsin-Madison\\
\texttt{lren29@wisc.edu}
\and
Christos Tzamos\thanks{Supported in part by NSF Award  CCF-2144298 (CAREER).}\\
University of Athens and Archimedes AI\\
\texttt{ctzamos@gmail.com}
}

\maketitle

\begin{abstract}
    We study the task of Multiclass Linear Classification (MLC) 
in the distribution-free PAC model 
with Random Classification Noise (RCN). 
Specifically, the learner is given a set of 
labeled examples $(x, y)$, where $x$ is drawn 
from an unknown distribution on $\R^d$ 
and the labels are generated by a 
multiclass linear classifier corrupted with RCN. 
That is, the label $y$ is flipped from $i$ to $j$ 
with probability $H_{ij}$ 
according to a known noise matrix $H$ with 
non-negative separation 
$\sigma: = \min_{i \neq j} H_{ii}-H_{ij}$. 
The goal is to compute a hypothesis with 
small 0-1 error. For the special case of two labels, 
prior work has given polynomial-time algorithms 
achieving the optimal error. 
Surprisingly, little is known about 
the complexity of this task even for three labels.
As our main contribution, we show that the complexity 
of MLC with RCN becomes drastically different 
in the presence of three or more labels. 
Specifically, we prove {\em super-polynomial} 
Statistical Query (SQ) lower bounds for this problem. 
In more detail, even for three labels and 
constant separation, 
we give a super-polynomial lower bound 
on the complexity of any SQ algorithm achieving optimal error. 
For a larger number of labels  and smaller separation, 
we show a super-polynomial SQ lower bound even 
for the weaker goal of achieving {\em any} constant factor approximation to the optimal loss or even beating the trivial hypothesis.
    \end{abstract}

\setcounter{page}{0}
\thispagestyle{empty}
\newpage

\section{Introduction}
A multiclass linear classifier is any function 
$f: \R^d \to [k]$ of the form 
$f(x) = \argmax_{i \in [k]} (w_i\cdot x)$, 
where 
$w_i \in \R^{d}$ for all $i \in [k]$. 
(If the maximum is achieved by more than one indices, 
the tie is broken by taking the smallest index.). 
Multiclass Linear Classification (MLC)---the task of learning 
an unknown linear classifier from random labeled examples--- 
is a textbook machine learning problem~\cite{shalev2014understanding}, 
which has been extensively studied both theoretically and empirically 
\cite{platt1999large,hsu2002comparison,aly2005survey,duan2005best, 
tewari2007consistency,kakade2008efficient,huang2011extreme,beygelzimer2019bandit}. 
MLC naturally arises in a range 
of critical applications, including face recognition 
\cite{lihong2009face}, cancer diagnosis \cite{panca2017application}, 
ecological indicators \cite{bourel2018multiclass} and more. 
In all these settings, the number of labels is much larger than 
two---hence they cannot be modeled by binary linear classification. 
Moreover, MLC has important connections with modern deep learning 
architectures; the last layer of a neural network is typically 
a softmax function layer---a natural extension of MLC. 

The sample complexity of MLC is fairly well-understood in the 
PAC model, even in the presence of noise. 
Specifically, standard arguments, see, 
e.g.,~\cite{shalev2014understanding}, 
give that $\poly(d, k, 1/\eps)$ samples information-theoretically
suffice to achieve 0-1 error $\opt+\eps$, 
where $\opt$ is the optimal error achievable 
by any function in the class. Yet the computational complexity 
of this task has remained perplexing.
In the realizable setting (i.e., in the presence of clean labels), 
the sample complexity of MLC is $\Tilde{O}(dk/\eps)$ and the 
problem is solvable in polynomial-time via 
a reduction to linear programming (LP). Notably, 
the corresponding LP can be solved efficiently in the 
Statistical Query (SQ) model~\cite{Kearns:98}, 
using a rescaled Perceptron algorithm~\cite{dunagan2004simple}. 
Alas, the realizable setting is highly idealized---in most 
practical applications, some form of partial 
label contamination is unavoidable. 
It is thus natural to ask what is algorithmically possible 
in the presence of label noise.

If the label noise is adversarial~\cite{Haussler:92, KSS:94}
or even semi-random~\cite{Massart2006}, strong computational hardness 
results are known even for {\em binary} linear classification, corresponding to $k=2$~\cite{daniely2016complexity,diakonikolas2022cryptographic,diakonikolas2022near,nasser2022optimal,diakonikolas2023near,tiegel2023hardness}. 
These hardness results, of course, carry over to the multiclass setting.

Interestingly, if the label noise is random---formalized 
by the Random Classification Noise (RCN) 
model~\cite{AL88}---the binary linear classification task 
admits a polynomial-time algorithm. The first such 
algorithm was given in~\cite{blum1998polynomial}; 
see also~\cite{dunagan2004simple,diakonikolas2021forster,diakonikolas2023strongly}. 
All these algorithms are known to fit the SQ model. 

The preceding discussion motivates the algorithmic study of MLC 
in the presence of RCN. 
A positive algorithmic result for the multiclass case 
with RCN would be of significant theoretical and practical interest. {\em As our main contribution, we give strong evidence that such an efficient algorithm does not exist.}

To formally state our contributions, we require 
the definition of multiclass classification 
with RCN that has been widely studied in 
prior works~\cite{patrini2017making,van2018theory,ghosh2017robust}.

\begin{definition}[Multiclass Classification with RCN]
Let $X$ be the space of examples and let $Y=[k]$ be the label space.  
A multiclass classifier is any function $f: X \to Y$.
A noise matrix $H \in [0,1]^{k \times k}$ is a row stochastic matrix such that for every $i \in [k]$, $\sum_{j=1}^k H_{ij} =1$. An instance of multiclass classification with RCN is parameterized 
by $(D,f^*,H)$, where $f^*$ is the ground truth multiclass classifier, $H$ is the noise matrix, and $D$ is a joint 
distribution over $X \times Y$ such that each labeled example $(x,y)\sim D$ is generated as follows. We have $x \sim D_X$, where $D_X$ is the marginal of $D$ over $X$.  
The label $y$ of $x$ is drawn from the distribution such that $\Pr(y=j \mid x) = H_{f^*(x)j}$ for $j \in [k]$. 
The error of a hypothesis $h: X \to Y$ is defined as 
$\err(h) := \Pr_{(x,y)\sim D}\left(h(x) \neq y\right)$. 
Given a set $S$ of \iid examples drawn from $D$, 
$\epsilon \in (0,1)$, and a function class $\mathcal{F}$ 
such that $f^* \in \cal F$, 
a learner is asked to output a hypothesis $\hat{h}$ such that $\err(\hat{h}) \le \opt+\epsilon$, where $\opt:= \min_{f \in \cal F} \err(f)$. 
\end{definition}

We will also consider algorithms with 
approximate error guarantees, namely aiming for 0-1 
error $C \opt+\epsilon$, where $C$ is a universal constant.
While the focus of this paper is on the setting that 
$\mathcal{F}$ is the class of multiclass {\em linear} 
classifiers, the broader class of multiclass 
polynomial classifiers will arise in our proof. 
A multiclass degree-$m$ polynomial classifier 
$f_P: \R^d \to [k] $ is characterized 
by a collection $P=(p_1,\dots,p_k)$, 
where $p_i(x): \R^d \to \R, i \in [k]$, 
is a polynomial of degree at most $m$. 
For $x \in \R^d$, 
$f_P(x) = \argmax_{j \in [k]} p_j(x)$. 
(If the maximum is achieved by more than one indices, 
the tie is broken by taking the smallest index.)

If $\sigma \eqdef \min_{i\neq j}H_{ii}-H_{ij}\ge 0$, 
MLC with RCN can be solved up to error $\opt+\eps$ 
with sample complexity 
$\min\{\Tilde{O}(dk/\sigma \eps),\Tilde{O}(dk/\eps^2)\}$ 
via empirical risk minimization (ERM). 
(Moreover, the ground truth $f^*$ achieves the optimal 
0-1 loss of $\opt$.)
This 
can be directly deduced from
\cite{Massart2006}. Though the sample complexity of the problem is understood, its computational complexity 
has remained open even for the case of $3$ labels. 

A long line of work~\cite{wang2017multiclass,patrini2017making,van2018theory,lipton2018detecting,zhang2021learning} has focused on 
understanding MLC, and more general multiclass classification,  
with RCN from an algorithmic perspective---both theoretically 
and empirically. The methods proposed and analyzed in 
these works require inverting the noise matrix $H$ 
during the training process, and achieve sample and time 
complexity scaling inverse polynomially with the 
minimum singular value of $H$. This quantity could be 
arbitrarily small or even zero---
even if $k=3$ and separation $\sigma=0.1$.  
Hence, such approaches do not 
lead to an efficient algorithm in general.


For binary classification, \cite{Kearns:98} showed that 
any efficient SQ algorithm that succeeds 
in the realizable setting can be efficiently 
converted into an efficient SQ algorithm that solves the same problem in the presence of RCN. 
Unfortunately, no such result is known for the multiclass setting. This suggests that an efficient algorithm for MLC 
with RCN would require novel techniques.





Perhaps surprisingly, here we provide strong evidence that such an efficient algorithm does not exist, 
even for the case of 3 labels. Formally, we establish 
the first super-polynomial SQ lower bounds 
for MLC with RCN, suggesting that 
the complexity of the problem 
dramatically changes for $k\ge 3$.

SQ algorithms are a class of algorithms
that, instead of having direct access to samples, 
are allowed to query expectations of bounded functions of the distribution (see Definition~\ref{def:sq}). 
The SQ model was introduced in~\cite{Kearns:98}. 
Subsequently, the model has been extensively studied 
in a variety of contexts~\cite{Feldman16b}.
The class of SQ algorithms is broad and is known to 
capture a range of  
algorithmic techniques~\cite{FeldmanGRVX17, FeldmanGV17}.

All our hardness results hold even if the noise matrix 
$H$ is known and given as input to the learner.
Our first main result pertains to the case of optimal error
guarantee, and holds even if $\sigma$ is a positive constant. In particular, we show: 

\begin{theorem}[Informal Statement of \Cref{th additive}] \label{thm:exact-intro}
    There is a noise matrix $H \in [0,1]^{3\times3}$ with $H_{ii}-H_{ij} \ge 0.1, \forall i \neq j \in [3]$, such that it is SQ-hard to learn an MLC problem on $\R^d$, with RCN specified by $H$, 
    up to error $\opt+\eps$.
\end{theorem}

Given Theorem~\ref{thm:exact-intro}, 
it is natural to ask whether it is possible to {\em approximately} efficiently learning MLC under RCN. 
A learner is said to be 
$C$-approximate if it outputs a hypothesis $\hat{h}$ 
such that $\err(\hat{h}) \le C\opt + \eps, \forall \eps \in (0,1)$, where $C>1$. 
By considering larger values of $k$ and small separation $\sigma$, we show that even approximate learning 
is SQ-hard (even if the noise level is small). In more detail, we show: 
\begin{theorem}[Informal Statement of \Cref{cor approximation}]
    For any $C>1$, there exists a noise matrix $H\in [0,1]^{k\times k}$, where $k=O(C)$ and $\sigma=\min_{i,j} H_{i,i}-H_{i,j}=\Omega(1/C)$, 
    such that it is SQ-hard to learn an MLC problem on $\R^d$ with RCN specified by $H$ 
    up to error $C\opt$, even if $\opt=\Theta(1/C)$.
\end{theorem}

In fact, our final result shows that within the regime $\sigma=\min_{i \neq j} H_{ii}-H_{ij}>0$, 
it is even SQ-hard to learn a hypothesis 
with error $1-1/k-o(1)$, even if $\opt=O(1/k)$. 
That is, it is hard to even find a hypothesis 
better than guessing the labels uniformly at random. 
Specifically, even if we only add $1\%$ RCN 
for an instance of MLC, it is hard to learn a hypothesis with error better than $99\%$ for $k=100$.
\begin{theorem}[Informal Statement of \Cref{cor beat constant}]
     For any $k\in \Z_+$ with $k\geq 3$, there is a noise matrix $H\in [0,1]^{k\times k}$ with 
     $\sigma=\min_{i,j} H_{i,i}-H_{i,j}=\Omega(1/d)$ such that it is SQ-hard to learn an MLC problem on 
     $\R^d$ with RCN specified by $H$ 
     up to error $1-1/k-o(1)$, even if $\opt=O(1/k)$.
\end{theorem}

\begin{remark}
{\em Our SQ lower bounds imply similar lower bounds for 
low-degree polynomial tests~\cite{HS17,HKP17,Hop18}, via
a result of~\cite{brennan2021statistical}.}
\end{remark}

\section{Technical Overview} \label{sec:techniques}

At a high level, our proof leverages the SQ lower bound framework 
developed in \cite{DKS17-sq} (see also~\cite{DKPZ21, DKRS23}) 
and techniques for constructing 
distributions that match Gaussian moments 
from \cite{diakonikolas2022near} and \cite{nasser2022optimal}. 
We stress that while these ingredients are useful in 
our construction, employing them in our context 
requires novel conceptual ideas. 

Roughly speaking, prior work \cite{DKS17-sq} 
and its generalization \cite{diakonikolas2022near}
give a generic framework for proving SQ-hardness results for 
supervised learning problems. Consider a distribution $D$ over 
$\R^d \times [k]$---a distribution consistent with an instance of a 
classification problem. Suppose that for $y \in [k]$, 
the distribution of $x$ conditioned on $y$ 
has the form of $P^A_v$, for $v$ a hidden direction in $\R^d$, 
such that: (i) $x_v$, the projection of $x$ on the $v$ direction, 
follows a one-dimensional distribution $A$, and 
(ii) $x_{v^\perp}\sim N(0,I)$. Moreover, suppose 
that the one-dimensional distribution $A$ 
nearly matches the first $t$ moments of $N(0,1)$, 
within error $\nu$, and has  chi-squared norm at most $\beta$.
The any SQ algorithm that correctly distinguishes between 
the case where $(x, y)$ is drawn from such a $D$ versus 
the case where the label $y$ is generated independently of $x$ 
according to $D_y$, needs either to make $2^{\Omega(d)}$ statistical queries or to make a query with tolerance $2\sqrt{\tau}$, where $ \tau=\nu^2+2^{-\Omega(t)}\beta$.

To leverage the aforementioned result, it suffices for 
us to construct a distribution $D$ with the form discussed above 
that is consistent with an instance of MLC with RCN. 
Unfortunately, constructing such a $D$ directly is 
technically challenging. 
To overcome this obstacle,
we first note that one can reduce 
learning multiclass degree-$m$ {\em polynomial} classifiers 
with RCN to learning linear classifiers with RCN 
using the Veronese mapping,  
defined as $V(x)\eqdef (x,1)^{\otimes m}$.
Therefore, it suffices for us to give a distribution $D$ 
that is consistent with an instance of a multiclass polynomial classifier $f^*=\argmax\{p_1(v\cdot x),\dots,p_k(v\cdot x)\}$ 
with RCN over $\R^N$ instead. 
We remark that as we want to prove SQ hardness for MLC 
over $\R^d$ and $d=N^{O(m)}$, if $m$ is large, 
then we need to prove a stronger hardness result 
for the polynomial classification problem over $\R^N$. 
As we will discuss in \Cref{sec hard main}, a key towards 
proving our SQ-hardness result is to choose 
the correct value of $m$.


Given the above discussion, it suffices to look 
at one-dimensional distributions along $v$. 
Our key observation, that leads to our SQ-hardness result, 
is that for a noise matrix $H$, if the $k$th row $h_k$ 
can be written as a convex combination 
$h_k=\sum_{j \in [k-1]}a_jh_j$ of the other rows, 
then an example drawn from the marginal distribution 
$\sum_{j\in[k-1]}a_j P^{A_j}_v$, 
where $P^{A_j}_v$ is the marginal distribution for $x$ 
with $f^*(x)=j$, will have observed label following 
the distribution $h_k$. Intuitively, if $D$ has such a marginal 
distribution of $x$, then the conditional distribution 
on $y\in[k]$ should be a mixture of the base distributions $P^{A_j}_v$. As long as $A_i, i \in [k-1]$ is close to $N(0,1)$, 
this in turn can be shown to imply a hardness result. 

Such an intuition is useful but is not formally correct: 
in general, the conditional distribution is not exactly 
a mixture of the base distributions. We overcome this difficulty 
by mixing examples with $f^*(x)=i$ and $f^*(x)=k$ 
carefully to obtain $A_i, i\in [k-1]$. 
Specifically,  we leverage techniques for constructing 
moment-matching distributions from 
\cite{diakonikolas2022near, nasser2022optimal}.
Roughly speaking, these works 
considered the following distribution.
Let $G_{\delta,\xi}$ be the distribution of $z\sim N(0,1)$ 
conditioned on $z\in I_i=[i\delta-\xi,i\delta+\xi]$ for $i\in \Z_+$
i.e., conditioned on equally spaced intervals with   
width $\xi$.
The distribution $G_{\delta,\xi}$ approximately matches moments 
up to degree $1/\delta^2$ up to error $2^{-\Omega(1/\delta^2)}$ with $N(0,1)$ and its chi-squared norm is not too large 
for any $\xi\geq 2^{-\Omega(1/\delta^2)}$. We choose $A_1=G_{\delta,\xi}$. 
Inspired by \cite{nasser2022optimal}, for $i\in[-m,m]$, 
if we make tiny shifts for $A_1$ over these $I_i$, 
$k-2$ times to get $k-1$ unions of disjoint intervals 
$J_j, j\in [k-1]$, we obtain distributions $A_1,\dots,A_{k-1}$ 
that are close to $N(0,1)$. Importantly, one can construct 
polynomials $p_j(z)>0$ if and only if $z \in J_j, j \in [k-1]$, 
and $p_k(z)>0$ if and only if $z \not \in I_{in}:=\textbf{conv}\bigcup_j J_j$. 
This implies that $D$ is consistent with an instance of multiclass 
polynomial classification with degree $O(m)$. 

Moreover, we show that the larger $\Pr(z \in I_{in})$ is 
constructed, the better learning guarantee we can rule out. If we 
choose $\Pr(z \in I_{in})=\eps$, then we are able to rule out 
learning algorithms that achieve error $\opt+\eps$. Moreover, by 
carefully designing the noise matrix $H$ and choosing 
$\Pr(z \in I_{in})=1-1/\poly(k)$, we are additionally 
able to rule out algorithms with error better 
than $1-1/k$, even if $\opt=O(1/k)$.

\section{Preliminaries} \label{sec:prelims}

\noindent {\bf Notation} 
Let $f^*:X \to Y$ be the ground truth hypothesis. 
For $j \in [k]$, denote by $S_j=\{x \mid f^*(x) = j\} \subseteq \R^d$ the set of examples with $f^*(x) = j$. 
Let $h:X \to Y$ be an arbitrary hypothesis. 
For $i,j \in [k]$, we denote by $S_{ji}=\{x \mid f^*(x) = j, h(x) = i\}\subseteq \R^d$, the set of examples with ground 
truth label $j$, but on which $h$ predicts $i$. In this paper, 
we use $\mathbb{S}^{d-1}$ to denote the unit sphere in $\R^d$.
Let $K \subseteq \R^d$ be any set; 
we denote by $\textbf{conv}(K)$, the convex hull of $K$. 
For a noise matrix $H \in [0,1]^{k \times k}$, 
we denote by $h_i, i \in [k]$, the $i$th row vector of $H$. 

For a distribution $D$, we use $\E_{\bx\sim D}(x)$ to denote the expectation of $D$. Let $D$ be a distribution of $(x,y)$ over $\R^d \times [k]$. We use $D_X$ to denote the marginal distribution of $D$ over $\R^d$ and use $D_y$ to denote the marginal distribution of $D$ over $y$. In this paper, we will use $N(0,I)$ to denote the standard Gaussian distribution over $\R^d$ and use $N(0,1)$ to denote the standard one-dimensional normal distribution. For $N(0,1)$, we use $G(x)$ to denote its density function and use $\gamma_t, t \in \N$, to denote its standard $t$th moment $\E_{x\sim N(0,1)} x^t$.

\noindent{\bf Basics of the SQ Model}
We record here the formal definition of the model 
and the definition of pairwise correlation. More detailed
background can be found in \Cref{app pre}.

\begin{definition}[SQ Model] \label{def:sq}
Let $D$ be a distribution over $X \times Y$. 
A \emph{statistical query} is a bounded function $q:X \times Y \rightarrow[-1,1]$. 
We define $\mathrm{STAT}(\tau)$ to be the oracle that given any such query $q$, outputs a value $v$ such that $|v-\E_{(x,y)\sim D}[q(x,y)]|\leq\tau$, where $\tau>0$ is the \emph{tolerance} parameter of the query.
A \emph{statistical query (SQ) algorithm} is an algorithm 
whose objective is to learn some information about an unknown 
distribution $D$ by making adaptive calls to the corresponding $\mathrm{STAT}(\tau)$ oracle.
\end{definition}

\begin{definition} [Pairwise Correlation]
The pairwise correlation of two distributions with probability density function $D_1, D_2:\R^d\mapsto \R_+$ 
with respect to a distribution with density $D:\R^d\mapsto \R_+$, 
where the support of $D$ contains the support of $D_1$ and $D_2$, 
is defined as $\chi_D(D_1,D_2)\eqdef \int_{\R^d}D_1(\x)D_2(\x)/D(\x)d\x-1$. 
Furthermore, the $\chi$-squared divergence of $D_1$ to $D$ is defined as
$\chi^2(D_1,D)\eqdef \chi_D(D_1,D_1)$.
\end{definition}

\paragraph{Organization} 

The structure of this paper is as follows:
In \Cref{sec test to learn}, we introduce an appropriate 
hypothesis testing problem and show how to efficiently reduce it 
to our learning problem. 
In \Cref{sec hard testing}, we construct the hard distribution for
the testing problem.
Finally, in \Cref{sec hard main}, we prove the main results 
of this paper by carefully choosing the parameters 
and putting together the reduction of \Cref{sec test to learn} 
and the hard distribution construction of \Cref{sec hard testing}.

\section{From Hypothesis Testing to Learning}\label{sec test to learn}

The usual way to prove an SQ-hardness result 
for a learning problem is to show hardness for an appropriate
hypothesis testing problem that efficiently reduces to 
the learning problem. 
In this section, we explore the properties of multiclass classification problems and introduce the hypothesis testing problem. 
To start with, we give the following condition, which will be used to create a hard hypothesis testing problem related to the multiclass classification problem.

\begin{definition}[SQ-Hard to Distinguish Condition]\label{def condition}
    Let $I=(D,f^*,H)$ be an instance of multiclass classification 
    with RCN. We say that the instance $I$ satisfies the hard-to-distinguish condition if it satisfies the following conditions:  
    \begin{enumerate}[nosep]
        \item \label{cond convex} There exist $ a_i \ge 0, i \in [k-1], \sum_{i=1}^{k-1} a_i = 1$ such that $h_k = \sum_{i=1}^{k-1} a_i h_i$.
        \item \label{cond margin} $\forall i \in [k-1], \Pr_{x \sim D_X}\left( x \in S_i \mid x \neq S_k  \right) = a_i$.
        \item \label{cond opt} $H_{jj}-H_{ji} \ge 0, \forall j,i \in [k]$. 
    \end{enumerate}
\end{definition}

We give some intuition behind \Cref{def condition}. Given a noise matrix $H$, the first condition, \cref{cond convex}, implies that the $k$th row vector of $H$ can be written as a convex combination of $h_i, i \in [k-1]$, with convex coefficients $a_i$. Recall that an example $x$ with ground truth label $i$ has $H_{ij}$ probability to be observed as label $j$. If the probability mass of $S_1,\dots,S_{k-1}$ are assigned proportionally to these convex coefficients (\cref{cond margin}), then a random example drawn from these regions will have probability $H_{ki}$ to have label $i$. Specifically, if we draw a random example $x\sim D_X$, then
\begin{align*}
 \Pr_{(x,y)\sim D}(y = i) 
   &  = \Pr_{x\sim D_X}(x \not \in S_{k}) \littlesum\nolimits_{j=1}^{k-1} a_j H_{ji} + \Pr_{x\sim D_X}(x \in S_{k}) H_{ki} \\
    &= \Pr_{x\sim D_X}(x \not \in S_{k}) H_{ki} + \Pr_{x\sim D_X}(x \in S_{k}) H_{ki} = H_{ki}.
\end{align*}
This implies that the marginal distribution $D_y$ 
follows the discrete distribution $h_k$. 
This suggests that given a set of examples drawn \iid from $D$, 
without exploring the structure of the data, 
it is hard to tell whether the labels $y$ are generated 
from an instance of multiclass classification 
or generated from the distribution $h_k$ 
independent of $x$.

Motivated by this observation, we define the following testing problem,
which we will show later to be hard to solve if some additional 
distributional assumptions are satisfied.
\begin{definition}[Correlation Testing Problem]\label{def correlation testing}
A correlation testing problem $\mathcal{B}(D_0,\D)$ is defined by a distribution $D_0$ and a family of distributions $\D$ over $X \times Y$. An algorithm is given SQ query access 
to some distribution $D$ and a noise matrix $H$, and is asked to distinguish whether $D=D_0$ or $D \in \D$. In particular, $D_0$ and $\D$ satisfy the following properties.
\begin{enumerate} 
\item  Null hypothesis: $x \sim (D_0)_X$ for some known distribution $(D_0)_X$, where $y$ is independent of $x$ such that $\Pr_{y \sim (D_0)_y}(y = i) = H_{ki}, \forall i \in [k]$.
\item  Alternative hypothesis: 
$D\in \D$, where each distribution $D \in \D$ corresponds to a multiclass classification instance $(D,f^*,H)$ that satisfies \Cref{def condition}. 
\end{enumerate}
\end{definition}

The correlation testing problem asks an algorithm 
to test whether the distribution of $y$ 
is generated according to an instance of a multiclass classification problem or is generated from a known distribution that is independent of $x$. 
In the rest of this section, we establish the 
connection between the testing problem of 
\Cref{def correlation testing} and the multiclass classification problem. We first present the following error-decomposition lemma, that describes the error of any multiclass hypothesis $h$. We defer 
the proof of \Cref{lm error decomposition} 
to \Cref{app error decomposition}.

\begin{lemma}\label{lm error decomposition}
    Let $(D,f^*,H)$ be any instance of multiclass classification with RCN. Let $h: X \to Y$ be an arbitrary multiclass hypothesis over $X$. Then,
    \[
        \err(h) 
        = \sum_{j=1}^k  \Pr(S_j)(1-H_{jj}) + \sum_{i \neq j} \Pr(S_{ji})(H_{jj}-H_{ji})\; .
     \]
In particular, if $H_{jj}-H_{ji} \ge 0$ for every $j \in [k], i \neq j$, then 
\[
    \opt = \err(f^*) = \sum_{j=1}^k \Pr(S_j)(1-H_{jj}).
\]
\end{lemma}


Given \Cref{def correlation testing}, the following lemma reduces the correlation testing problem to the classification problem.



\begin{lemma}
    \label{th test-to-learn}
Let $\D$ be a family of distribution over $X \times Y$ such that each distribution $D \in \D$ corresponds to a multiclass classification instance $(D,f^*,H)$ that satisfies \Cref{def condition}. If there is an SQ learning algorithm $\mathcal{A}$ such that for every instance $(D,f^*,H), D \in \D$, $\mathcal{A}$ makes $q$ queries, each of tolerance $\tau$, 
    and outputs a hypothesis $\hat{h}$ such that $\err(\hat{h}) \le \opt+\alpha$, where 
       $ 2\alpha = \sum_{j=1}^{k-1} \Pr(S_j)(H_{jj}-H_{jk}),$
    then there is an SQ algorithm $\mathcal{A}'$ that solves the correlation testing problem defined in \Cref{def correlation testing} by making $q+1$ queries, 
    each of tolerance $\min(\tau, \alpha/2)$.
\end{lemma}
We defer the full proof of \Cref{th test-to-learn} to \Cref{app test-to-learn lm} and give an overview of the proof 
below. Recall that the goal of the correlation testing problem is to tell whether the label $y$ is generated according to the discrete distribution $h_k$ independent of $x$ (null hypothesis) or is generated according to some $D$ from $\D$ (alternative hypothesis). In the former case, no hypothesis will have an error better than the constant hypothesis $h(x) \equiv k$, which has an error of $1-H_{kk}$. 
In the latter case, by \Cref{lm error decomposition}, we can show that $\err(k)-\opt = 2\alpha$. This implies that, if we can learn the multiclass classification problem up to error $\opt+\alpha$, then we only need to make a single query with tolerance $\alpha/2$ to check whether the hypothesis $\hat{h}$ we learn has error less than $1-H_{kk}-\alpha/2$ to solve the testing problem.

Given \Cref{th test-to-learn}, we briefly explain how 
we will use it to prove the hardness results. 
Notice that if we choose $H_{jj}-H_{jk} \ge c$, 
for some constant $c>0$ and $j \in [k-1]$, 
then $\alpha = \Theta(1-\Pr(S_k))$. 
Thus, \Cref{th test-to-learn} suggests that in this case, 
to prove the hardness of the learning problem, 
it suffices for us to construct instances of a multiclass classification problem 
that satisfies \Cref{def condition} 
and $1-\Pr(S_k)$ is as large as the desired accuracy $\alpha$. 
In particular, the larger $1-\Pr(S_k)$ can be constructed, 
the stronger hardness result we are able to obtain.
We defer the details to \Cref{sec hard testing}.

\section{Hardness of Hypothesis Testing}\label{sec hard testing}
Given \Cref{th test-to-learn}, we know that it suffices 
to construct a correlation testing problem 
$\mathcal{B}(D_0,\mathcal{D})$ that is hard to solve. 
Notice that any multiclass polynomial classification problem 
can be represented as a multiclass linear classification problem 
in a higher dimension via the kernel method. 
So, we will consider constructing a family of correlation testing 
problems $\mathcal{B}(D_0,\mathcal{D})$ for multiclass polynomial 
classification problems. Consider an instance $I=(D,f^*,H)$ of 
multiclass polynomial classification problem with degree-$m$ 
under the SQ-hard to distinguish condition, 
where $f^*$ is characterized by
$k$ degree-$m$ polynomials of the form 
$p_1(v\cdot x),\dots,p_k(v\cdot x)$, $v \in \mathbb{S}^{d-1}$. 
Then the label $y$ is completely dependent on $v$, 
which implies that a learner must look at examples close to $v$ 
in order to solve the testing problem. 
This motivates us to look at the following hidden direction distribution that is frequently 
used in the literature of SQ lower bounds~\cite{DKS17-sq}.
\begin{definition}[Hidden Direction Distribution]
    Let $A$ be a distribution over $\R$, with probability density $A(x)$ and $v \in \R^d$ be a unit vector. Define $P^A_v(x) = A(v\cdot x)\exp\left(-\norm{x-(v\cdot x)v}^2_2/2 \right)/(2\pi)^{(d-1)/2}$, i.e., $P_v^A$ is a product distribution whose orthogonal projection onto $v$ is $A$ and onto the subspace orthogonal to $v$ is the standard $(d-1)$-dimensional Gaussian distribution.
\end{definition}

Based on the definition of the hidden direction distribution, we construct the family of hard distributions $\D$ as follows.

\begin{figure}
    \centering
    \includegraphics[width=0.7\linewidth]{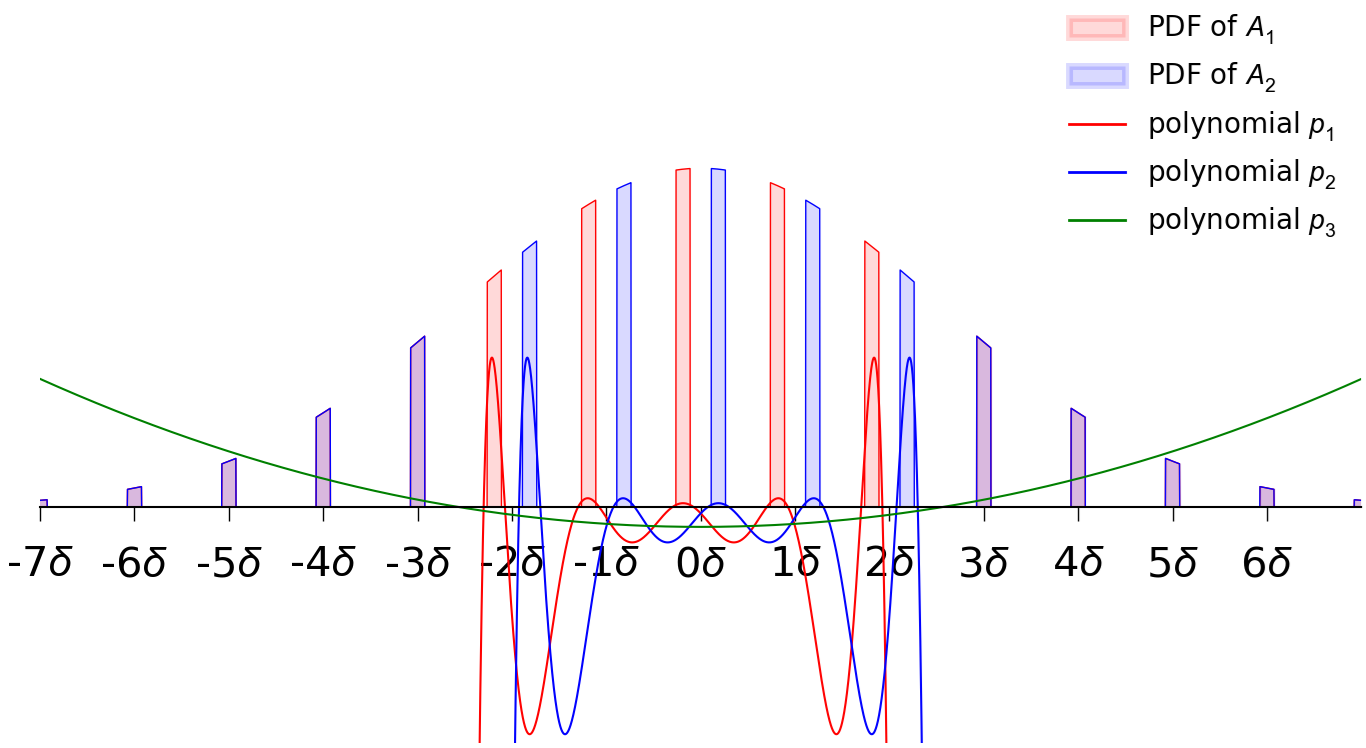}
    \caption{Illustration of base distributions for $k=3$. Histograms that are colored in red (resp. blue) correspond to distribution $A_1$ (resp. $A_2$). $p_1,p_2,p_3$ colored in red, blue, and green are polynomials that characterize the target hypothesis $f^*$. $J_1$(resp. $J_2$) are red (resp. blue) intervals within the range $(-2\delta,2\delta)$, where examples have ground truth label $1$ (resp. $2$). Examples outside $J_1 \cup J_2$ have ground truth label $3$.}
    \label{fig: base}
\end{figure}


\begin{definition}[Hidden Direction Distribution Family]\label{def 1D}
Let $H \in [0,1]^{k \times k}$ be a noise matrix that satisfies \eqref{cond convex}, \eqref{cond opt} in \Cref{def condition} with the convex combination coefficients $a \in [0,1]^{k-1}$. Let $A=(A_1,\dots,A_{k-1})$ be a list of $k-1$ \emph{base distributions}, where for $i \in [k-1], A_i$ is a one-dimensional distribution that satisfies the following conditions: 
\begin{enumerate} 
    \item \label{prop disjoint} \label{prop realizable} $\exists$ a set of $m$ disjoint intervals $J_i, i \in [k-1]$, such that $A_i(x)>0$, for $x \in J_i$ and $A_i(x)=0$, for $x \in I_{in} \setminus J_i$, where $I_{in} = \textbf{conv} \bigcup_{j \in [k-1]}J_j$.
    \item \label{prop projection} $\forall x\in \R \setminus I_{in}$, $A_i(x)=A_j(x), \forall i,j \in [k-1]$. 
\end{enumerate}
We define the hidden direction distribution family $\D=\{D^{A,a}_v\}_{v \in \mathbb{S}^{d-1}}$ over $\R^d \times [k]$ such that $(x,y) \sim D_v^{A,a}$ is sampled as follows. With probability $a_j$, $x \sim P^{A_j}_v$. If $x \in J_j$, sample $y = i$ with probability $H_{ji}$, otherwise, sample $y=i$ with probability $H_{ki}$.
\end{definition}

We summarize the key properties of a hidden direction distribution 
family in \Cref{th hidden family}, the main theorem of this section. 
Due to space limitations, the full proof is deferred to \Cref{app hidden family}.

\begin{theorem}\label{th hidden family}
    Let $\mathcal{B}(D_0,\D)$ be a correlation testing problem, where $(D_0)_X = N(0,I)$ and $\D$ is a hidden direction distribution family. Suppose there exists $\nu>0$ such that for $\ell \le t \in \Z_+$, the family of one-dimensional distribution $A_1,\dots,A_{k-1}$ satisfies $ \abs{\E_{x \sim A_i} x^\ell - \gamma_\ell} \le \nu$. Then: 
    \begin{enumerate} 
        \item \label{prop realize} Every distribution $D^{A,a}_v \in \D$ is consistent with an instance of multiclass polynomial classification with RCN $(D^{A,a}_v,f^*,H)$ with degree at most $2m$ that satisfies \Cref{def condition}.
        \item \label{prop query} For any small enough constant $c>0$,
    let $\beta = \max_{i,j}\chi_{N(0,1)}(A_i,A_j)$ and let $\tau:= \nu^2 + c^t\beta$. Any SQ algorithm that solves $\mathcal{B}(D_0,\D)$ must make a query with accuracy better than $2\sqrt{\tau}$ or make $2^{\Omega_c(d)}\tau/\beta$ queries. 
    \end{enumerate}
\end{theorem}

In the rest of this section, we give an overview of the construction of the hidden direction distribution family as well as the proof of \Cref{th hidden family}. 

Consider a distribution $D^{A,a}_v \in \D$ (see \Cref{fig: base} for an example), where $\D$ is some hidden direction distribution family. 
\Cref{prop disjoint} in the construction of a hidden direction distribution family is to ensure each $D^{A,a}_v$ is consistent with some multiclass polynomial classification problem with degree $O(m)$.
Since for each $i \in [k-1]$, $J_i$ is a set of $m$ disjoint intervals, we know there is a degree-$2m$ polynomial $p_i(t): \R \to \R$ such that $p_i(t)>0$ if and only if $t \in J_i $. 
  On the other hand, since $I_{in} = \textbf{conv}\bigcup_{j \in [k-1]}J_j$ is a finite interval, there is a degree-2 polynomial $p_k(t): \R \to \R$ such that $p_k(t)>0$ if and only if $t \not\in I_{in}$. Since $J_i \cap J_j = \emptyset, \forall i \neq j$, we know that for each $j \in [k-1]$, if $v\cdot x \in J_j$, then $j=\argmax\{p_1(v\cdot x),\dots, p_{k}(v \cdot x)\}$ and if $v\cdot x \not\in I_{in}, k=\argmax\{p_1(v\cdot x),\dots, p_{k}(v \cdot x)\}$. Thus, $D^{A,a}_v$ is consistent with an instance of multiclass polynomial classification with RCN $(D^{A,a}_v,f^*,H)$, where the marginal distribution is $\sum_{j=1}^{k-1}a_j P_v^{A_j}$ and the ground truth hypothesis $f^*(x) = \argmax\{p_1(v\cdot x),\dots, p_{k}(v \cdot x)\}$. This gives an overview of the first part of \Cref{th hidden family}.

  We next focus on the second part of \Cref{th hidden family}. To simplify the notation, for each $i \in [k]$, we denote by $D^i_v = D^{A,a}_v(x \mid y =i)$ in the rest of this section. The proof strategy here is to use the standard SQ dimension (\Cref{lem:sq-lb}). 
  It is well-known that for any small constant $c$, there exists at least $2^{\Omega_c(d)}$ many unit vectors $u,v$ such that $\abs{u,v} \le c$ (see \Cref{lem:near-orthogonal}). Thus, to use \Cref{lem:sq-lb}, we only need to bound $\chi_{D_0}(D^{A,a}_v,D^{A,a}_u)$, where $\abs{u\cdot v} \le c$ as well as $\chi_{D_0}(D^{A,a}_v,D^{A,a}_v)$. Since
$       \chi_{D_0}(D^{A,a}_v,D^{A,a}_u) 
     = \sum_{i=1}^k H_{ki} \chi_{N(0,I)} \left(D^i_v,D^i_u\right),
$ 
 it is equivalent to upper bounding $\chi_{N(0,I)} \left(D^i_v,D^i_u\right)$. However, even though the base distributions $P^{A_j}_v$ are all close to $N(0,I)$ (thus have a small pairwise correlation), in general, conditional on the label $y=i$, there is no such guarantee for $D^i_v$. Here, we make use of \Cref{prop projection}, in the construction of the distribution family. Under \Cref{prop projection}, every $D^i_v$ is indeed a mixture of the distributions $P^{A_j}_v, j \in [k-1]$. Formally, we have the following technical lemma (see \Cref{fig: hard_main_sketch} for intuition), the proof of which is deferred to \Cref{app projection}.

\begin{figure}
    \centering
    \includegraphics[width=0.7\linewidth]{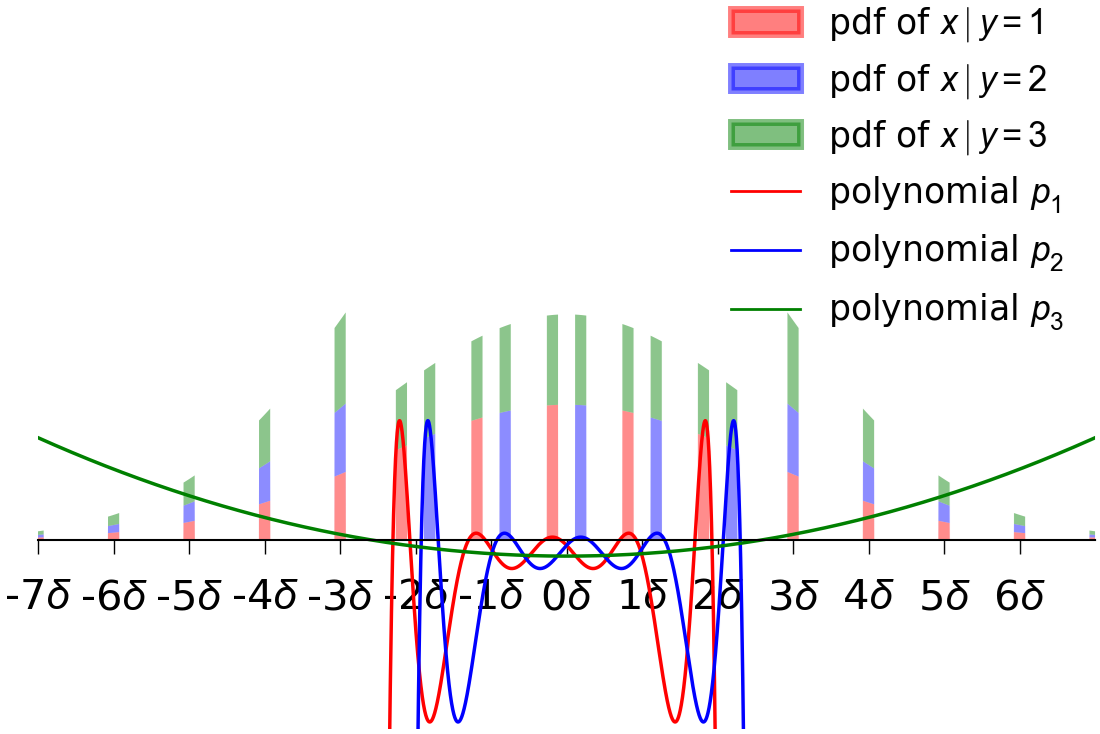}
    \caption{Illustration of $(D^{A,a}_v)_{\mid y=i}$ for $k=3$. $p_1,p_2,p_3$ colored in red, blue, green are polynomials that characterize the ground truth $f^*$. Histograms in red (resp. blue, green) correspond to distribution $(D^{A,a}_v)_{\mid y=1}$ 
    (resp. $(D^{A,a}_v)_{\mid y=2}$, $(D^{A,a}_v)_{\mid y=3}$). For each $i$, $(D^{A,a}_v)_{\mid y=i}$ has many moments close to the moments of the standard normal. }
    \label{fig: hard_main_sketch}
\end{figure}



\begin{lemma}[Distribution Projection]\label{lm projection}
    Let $\D$ be a hidden direction distribution family over $\R^d$ and let $D^{A,a}_v \in \D$ be a distribution that is consistent with an instance of multiclass polynomial classification with RCN $(D^{A,a}_v,f^*,H)$. For every $i \in [k]$, $D^{A,a}_v(x \mid y =i) = \sum_{j=1}^{k-1} \frac{a_j H_{ji}}{H_{ki}} P^{A_j}_v(x).  $
\end{lemma}

Given \Cref{lm projection}, to upper bound the pairwise correlation, it is equivalent to upper bound $\chi_{N(0,I)}(P^{A_j}_v,P^{A_j}_u)$. Using the correlation lemma developed in \cite{diakonikolas2022near}, as long as each $A_j$ has many moments close to those of a standard normal distribution, $\chi_{N(0,I)}(P^{A_j}_v,P^{A_j}_u)$ is small and we are able to prove hardness result using \Cref{lem:sq-lb}. This gives an overview of the second part of \Cref{th hidden family}.



\paragraph{Construction of Hard Distributions}
In the rest of the section, we will construct the family of hard distributions $A_1,\dots,A_{k-1}$. By \Cref{th hidden family}, we require $A_1,\dots,A_{k-1}$ to be supported in disjoint intervals and have many moments close to those of a standard normal distribution. The most natural construction of one such distribution is to restrict $N(0,1)$ over a sequence of discrete intervals. Related ideas have been used in proving hardness result
of various learning problems, such as \cite{bubeck2019adversarial,diakonikolas2022near,nasser2022optimal,tiegel24a}. By choosing such a distribution as $A_1$, we are able to construct $A_2,\dots, A_{k-1}$ by shifting the intervals constructed in $A_1$ with 
different step sizes. Such a construction can be viewed as a generalization of 
the technique in \cite{nasser2022optimal} for proving SQ-hardness of learning a halfspace under Massart noise. We present the formal construction 
of our hard distributions in \Cref{def distribution} and 
list its properties in \Cref{prop distribution}. We defer the proof of \Cref{prop distribution} to \Cref{app proof distribution}.

\begin{definition}\label{def distribution}
    For $\delta,\xi>0$ such that $\delta>4(k-1)\xi, 4(k-1)\xi<1$, we define 
\[G_{\delta,\xi} = \littlesum_{n \in \Z} \frac{\delta}{2\xi} G(x)\Ind(x \in [n\delta-\xi,n\delta+\xi])\] 
and $G^{(n)}_{\delta,\xi} = G_{\delta,\xi} / \norm{G_{\delta,\xi}}_1$, where $G(x)$ is the density of $N(0,1)$.
We define $A_1(x) = G^{(n)}_{\delta,\xi}$ and for $i \in [k-1]$, we define 
\begin{align*}
 A_i(x) =
\begin{cases}
     & A_1(x+4(i-1)\xi) \quad \abs{x} \le m\delta+(4i-3)\xi \\
     & A_1(x) \quad \abs{x} > m\delta+(4i-3)\xi,
 \end{cases}
\end{align*}
where $m \in \Z_+$. 
\end{definition}

\pagebreak

\begin{proposition}\label{prop distribution}
The univariate distributions $A_1,\dots,A_{k-1}$ constructed in \Cref{def distribution} satisfy
    \begin{enumerate}
    \item $\exists$ a set of $m$ disjoint intervals $J_i, i \in [k-1]$ such that $A_i(x)>0$, for $x \in J_i$ and $A_i(x)=0$, for $x \in I_{in} \setminus J_i, I_{in} = \textbf{conv}\bigcup_{j \in [k-1]}J_j$.
    \item $\forall x\in \R \setminus I_{in}, A_i(x)=A_j(x), \forall i,j \in [k-1]$.
        \item For $i,j \in [k-1]$, $\chi_{N(0,1)} \left(A_i,A_j\right) \le O(\delta/\xi)^2$.
        \item For $t \in \N$ and for $i \in [k-1]$, $ \abs{\E_{x \sim A_i} x^t - \gamma_t} \le O(t!)\exp(-\Omega(1/\delta^2))+ 4(k-1)\xi(1+2m\delta)^t.$
    \end{enumerate}
\end{proposition}

\section{Hardness of Multiclass Linear Classification Under RCN}\label{sec hard main}

In this section, we present our main hardness results for MLC. The proofs in 
this section use \Cref{th test-to-learn} to reduce  
the correlation testing problem (\Cref{def correlation testing}) 
to the MLC learning problem, and 
construct a hidden direction distribution family (\Cref{def 1D}) 
via the hard distribution defined in \Cref{def distribution} 
for the correlation testing problem. We will carefully 
choose the parameters for the hard distribution $A_1,\dots,A_{k-1}$ 
to invoke \Cref{prop distribution} and \Cref{th hidden family} 
to get SQ lower bounds for different learning guarantees.


Here we first present the intuition behind the proof of our hardness results. As we discussed in \Cref{sec hard testing}, to prove the hardness result, it is sufficient to construct a hidden direction distribution family $\D=\{D^{A,a}_v\}_{v \in \mathbb{S}^{N-1}}$ consistent with a multiclass polynomial classification instance in $\R^N$ such that the base distributions $A_1,\dots,A_{k-1}$ have many moments that are close to those of a standard normal. Recall that the construction of a hidden direction distribution family relies on a noise matrix $H$ that satisfies the SQ-hard to distinguish condition (\Cref{def condition}).
As an example, consider the noise matrix
\begin{align}\label{eq example}
    H = \begin{pmatrix}
        0.6 & 0 & 0.4 \\
        0 & 0.6 & 0.4 \\
        0.3 & 0.3 & 0.4
    \end{pmatrix},\; 
\end{align}
with $k=3$, 
where $h_3 = (h_1+h_2)/2$. Therefore, we choose base distributions $A_1,A_2$ constructed in \Cref{def distribution}, $a=(1/2,1/2)$ and $f^*(x)=\argmax\{ p_1(x),p_2(x),p_3(x)\}$ illustrated in \Cref{fig: base}. Recall that the goal of the correlation testing problem is to tell whether the label $y$ is generated according to the discrete distribution $h_k$ or is generated by some $D^{A,a}_v$. By \Cref{th test-to-learn}, for every $D \in \D$ 
$
  2\alpha:=  \err_D(k) - \opt = \sum_{j=1}^{k-1} (H_{jj}-H_{jk})\Pr(S_j),
$
the probability mass of $I_{in}$, the intervals in the middle as shown in \Cref{fig: hard_main_sketch}, with respect to $A_1$. 
This implies that the larger $\alpha$ is chosen, the better learning guarantee we can rule out. In particular, 
since $H_{jj}-H_{jk}, j \in [k-1]$ is larger than some universal constant, $\alpha$ is proportional to $\sum_{j=1}^{k-1} \Pr(S_j)$. By the construction of $A_1,\dots,A_{k-1}$, this quantity is exactly
$
    \Pr_{z \sim A_1}(z \in I_{in}) = \Pr_{z \sim A_1}(z \in I_{in})(\abs{z} \le (m+1)\delta).
$
Since $A_1$ is an approximation of a standard normal, 
the parameters $m,\delta$ of $A_1$ are selected such that 
$\Pr_{z \sim N(0,1)}(\abs{z} \le m\delta) \propto \alpha$. 
On the other hand, by \Cref{prop distribution}, 
for a given pair of $m,\delta$, by properly choosing small $\xi$, 
one can make the accuracy parameter $\tau$ in \Cref{th hidden family} 
as small as $\exp(-\poly(1/\delta))$. However, this does not imply 
that we can choose $\delta$ arbitrarily small for the following reason: 
to solve the polynomial classification problem in $\R^N$, 
we need to embed the instance to $\R^d, d=N^{O(m)}$, 
and solve it with an algorithm for MLC. Therefore, if $\delta$ is chosen 
too small, $m$ could be too large to rule out any hardness result 
for MLC. That is, a good tradeoff between $m,\delta$ 
is needed to prove our hardness result.

\paragraph{Hardness of Getting Error $\opt+\eps$}
Recall that for binary linear classifiers, 
there is an algorithm that runs in time $\poly(d,1/\epsilon)$ and 
outputs a hypothesis with error $\opt+\eps$. 
Moreover, the algorithm works even 
when $\min_{i \neq j} H_{ii}-H_{ij}=0$.
As it turns out, this is not the case for the multiclass case.
Our first SQ-hardness result shows that even if $k=3$ 
and $\min_{i,j} H_{ii}-H_{ij}>c$ for some constant $c$, 
to learn a hypothesis up to error $\opt+\epsilon$, 
one needs super-polynomial SQ complexity. 
Formally, we establish the following \Cref{th additive}, 
whose proof is deferred to \Cref{app additive}.

\begin{theorem}\label{th additive}
There is a matrix $H \in [0,1]^{3\times3}$ with $H_{ii}-H_{ij} \ge 0.1, \forall i \neq j \in [3]$, such that any algorithm $\mathcal{A}$ that distribution-free learns multiclass linear classifiers 
with RCN specified by $H$ on $\R^d$ to error $\opt+\epsilon$, 
$\epsilon \in (0,1)$, requires either (a) at least $d^{\Tilde{\Omega}(\log^{1.98}(d)/\eps^{1.98})}  $ queries, or
(b) a query of tolerance at most 
$1/d^{\Tilde{\Omega}(\log^{1.98}(d)/\eps^{1.98})}$. 
\end{theorem}

The proof of \Cref{th additive} follows the above intuition. 
To show that learning up to error $\opt+\eps$ is hard, 
we choose parameters $m,\delta$ such that $\Pr(I_{in})\approx \epsilon$. 
By the concentration properties of $N(0,1)$, we only need to choose 
$m\delta \approx \eps$. In the proof, we show that $m=\eps \sqrt{N}$ 
suffices to give a super-polynomial lower bound.

\paragraph{Hardness of Approximation and Beating Random Guess}
Given the hardness result in \Cref{th additive} of getting error 
$\opt+\epsilon$, one natural question is what kind 
of error guarantee we can efficiently achieve for MLC. 
For larger values of $k$ and small separation 
$\min_{i \neq j} H_{ii}-H_{ij}$, we show it is also hard 
to get any constant factor approximation, 
or even find a hypothesis with an error nontrivially better 
than a random guess given $\opt=O(1/k)$.
Formally, we first give the following theorem, whose proof is deferred to \Cref{app multiply}.

\begin{theorem}\label{th multiply}
For any $k\in \Z_+$ and $k\geq 3$,
there is a noise matrix $H\in [0,1]^{k\times k}$ such that $\max_{i,j} H_{i,i}-H_{i,j}=\zeta>0$ and has the following property:
For any sufficiently large $d\in \Z_+$,
any SQ algorithm $\A$ that distribution-free learns multiclass linear classifiers with RCN specified by $H$ on $\R^d$ to error $1-1/k-\zeta-2\mu$ requires either (a) at least $q$ queries, or
(b) a query of tolerance at most $\mu$,
where $\min(q,1/\mu^2)=d^{\Omega(\log^{0.99} d)}$.
In particular, this holds even if $\opt \le 1/k+\zeta+1/k^3$.
\end{theorem}

Given $k\in \Z_+$ and $k\geq 3$, we construct the corresponding noise matrix $H$ as
    \begin{align*}
    H = \begin{pmatrix}
        (k-1)/k-\zeta & 0 & \cdots & 1/k+\zeta \\
        0 & (k-1)/k-\zeta & \cdots & 1/k+\zeta \\
        \cdots & \cdots  & \cdots & \cdots\\
        1/k-\zeta/k & 1/k-\zeta/k & \cdots & 1/k+\zeta
    \end{pmatrix}.
    \end{align*}
Recall that for a hidden direction distribution family $\D$, 
$\err_D(k)=1-1/k-\zeta$, no matter if $D=D_0$ or $D\in \D$. 
Thus, if we can learn a hypothesis with error $1-1/k-\zeta-o(1)$, 
we are able to solve the correlation testing problem. 
To make this possible, we need to make $\opt$ as small as possible. 
By our construction, if an $x$ has ground truth label $f^*(x) \in [k-1]$, 
the probability that it is flipped is only $1/k+\zeta$. 
Thus, if we are able to choose 
$\sum_{j=1}^{k-1} \Pr(S_j)=\Pr(I_{in})=1-1/\poly(k)$, then $\opt=1/k+\zeta+1/\poly(k)$. By the tail bound of $N(0,1)$, 
to make this hold, we choose $m\delta=\Theta(\sqrt{\log k})$. 
Recall that we still need to make $m$, the degree of the polynomial we use, 
as small as possible. Here we choose $m=\Theta(\sqrt{N\log k})$, 
which suffices to give a super-polynomial SQ lower bound.


Given \Cref{th multiply}, we immediately obtain 
two corollaries for the hardness of approximate leaning 
and beating a random guess hypothesis respectively in the setting of multiclass linear classification. 
We defer the proofs of these two corollaries to \Cref{app corollary}.

\begin{corollary} [SQ hardness of Approximate Learning] \label{cor approximation}
    For any $C>1$, there exists a noise matrix $H\in [0,1]^{k\times k}$, with $k=O(C)$ and $\min_{i,j} H_{i,i}-H_{i,j}=\Omega(1/C)$ such that any SQ algorithm $\A$ that distribution-free learns multiclass linear classifiers on $\R^d$ with RCN parameterized by $H$ to error $C\opt$ given $\opt=\Omega(1/C)$ either 
    \begin{enumerate} 
        \item [(a)] requires at least $d^{\Omega(\log^{0.99} d)}$ queries, or
        \item [(b)] requires a query of tolerance at most $1/d^{\Omega(\log^{0.99} d)}$.
    \end{enumerate}
\end{corollary}

\begin{corollary} [SQ hardness of Beating Random Guess] \label{cor beat constant}
    For any $k\in \Z_+$ and $k\geq 3$, there is a noise matrix $H\in [0,1]^{k\times k}$ with $\min_{i,j} H_{i,i}-H_{i,j}=1/\poly(d)$ such that any SQ algorithm $\A$ that distribution-free learns multiclass linear classifiers on $\R^d$ with RCN parameterized by $H$ to error $1-1/k-1/\poly(d)$ given $\opt=O(1/k)$ either 
    \begin{enumerate} 
        \item [(a)] requires at least $d^{\Omega(\log^{0.99} d)}$ queries, or
        \item [(b)] requires a query of tolerance at most $1/d^{\Omega(\log^{0.99} d)}$,
    \end{enumerate}
\end{corollary}
It is worth noting that if we predict the label of an example with $y \in [k]$ uniformly at random, then the error is always $1-1/k$. Therefore, \Cref{cor beat constant} implies that it is not possible for an efficient SQ algorithm to output a hypothesis with error nontrivially better than a random guess hypothesis.

\section{Conclusion}
We conclude this paper with a conceptual implication of our 
results. Our SQ lower bounds exhibit the existence of
a very simple multi-index model that 
is easy to learn with perfect labels, 
but is hard to learn even 
with a small level of random label noise. 
Finally, we remark that the results of our 
work motivate several interesting directions, including the 
algorithmic study of MLC with more structured noise or 
structured marginal distributions. 
Recent work \cite{DIKZ25} made algorithmic progress in this 
direction for the case of Gaussian marginals.


\bibliographystyle{halpha-abbrv}
\bibliography{mydb}

\newpage

\appendix


\section*{Appendix}

The Appendix is organized as follows:
In \Cref{app pre}, we provide a complete list of 
preliminaries and give additional background on 
the Statistical Query model. In \Cref{app test to learn}, 
we present omitted proofs from \Cref{sec test to learn}. 
In \Cref{app hard testing} we present omitted proofs from 
\Cref{sec hard testing}, and in \Cref{apd sec hard main} we 
give missing proofs from \Cref{sec hard main}.

\section{Additional Preliminaries}\label{app pre}
\paragraph{Notation} 
Let $f^*:X \to Y$ be the ground truth hypothesis. For $j \in [k]$, denote by $S_j=\{x \mid f^*(x) = j\} \subseteq \R^d$ be the set of examples with $f^*(x) = j$. Let $h:X \to Y$ be an arbitrary hypothesis. For $i,j \in [k]$, we denote by $S_{ji}=\{x \mid f^*(x) = j, h(x) = i\}\subseteq \R^d$, the set of examples with ground truth label $j$, but on which $h$ predicts $i$. In this paper, 
we use $\mathbb{S}^{d-1}$ to denote the unit sphere in $\R^d$.
Let $K \subseteq \R^d$ be any set, we denote by $\textbf{conv}(K)$, the convex hull of $K$. 
For a noise matrix $H \in [0,1]^{k \times k}$, 
we denote by $h_i, i \in [k]$, the $i$th row vector of $H$. 

For a distribution $D$, we use $\E_{\bx\sim D}(x)$ to 
denote the expectation of $D$. Let $D$ be a distribution 
of $(x,y)$ over $\R^d \times [k]$. We use $D_X$ to denote 
the marginal distribution of $D$ over $\R^d$ and use $D_y$ 
to denote the marginal distribution of $D$ over $\{\pm 1\}$. 
In this paper, we will use $N(0,I)$ to denote the standard 
Gaussian distribution over $\R^d$ and use $N(0,1)$ to denote 
the standard one-dimensional normal distribution. For $N(0,1)$, 
we use $G(x)$ to denote its density function and use 
$\gamma_t, t \in \N$ to denote its standard $t$-th moment 
$\E_{x\sim N(0,1)} x^t$.

\paragraph{Background on the SQ Model}

Here we record the necessary background on the SQ model.

\begin{definition}[SQ Model] 
Let $D$ be a distribution over $X \times Y$. 
A \emph{statistical query} is a bounded function $q:X \times Y \rightarrow[-1,1]$. 
We define $\mathrm{STAT}(\tau)$ to be the oracle that given any such query $q$, outputs a value $v$ such that $|v-\E_{(x,y)\sim D}[q(x,y)]|\leq\tau$, where $\tau>0$ is the \emph{tolerance} parameter of the query.
A \emph{statistical query (SQ) algorithm} is an algorithm 
whose objective is to learn some information about an unknown 
distribution $D$ by making adaptive calls to the corresponding $\mathrm{STAT}(\tau)$ oracle.
\end{definition}

\begin{definition} [Pairwise Correlation]
The pairwise correlation of two distributions with probability density function $D_1, D_2:\R^d\mapsto \R_+$ 
with respect to a distribution with density $D:\R^d\mapsto \R_+$, 
where the support of $D$ contains the support of $D_1$ and $D_2$, 
is defined as $\chi_D(D_1,D_2)\eqdef \int_{\R^d}D_1(\x)D_2(\x)/D(\x)d\x-1$. 
Furthermore, the $\chi$-squared divergence of $D_1$ to $D$ is defined as
$\chi^2(D_1,D)\eqdef \chi_D(D_1,D_1)$.
\end{definition}

\begin{definition} [Statistical Query Dimension]
For $\beta,\gamma>0$, a decision problem $\mathcal{B}(\D,D)$, where $D$ is a fixed distribution and 
$\D$ is a family of distribution, let $s$ be the maximum integer such that
there exists a finite set of distributions $\D_D\subseteq\D$ such that 
$\D_D$ is $(\gamma,\beta)$-correlated relative to $D$ and $|\D_D|\geq s$. 
The Statistical Query dimension with pairwise correlations $(\gamma,\beta)$ of $\mathcal{B}$ is defined to be $s$, 
and denoted by $s=\mathrm{SD}(\mathcal{B},\gamma,\beta)$. We say that a set of $s$ distribution $\{D_1,\cdots,D_s\}$ over $\R^d$ is 
$(\gamma,\beta)$-correlated relative to a distribution $D$ if $\chi_D(D_i,D_j)\leq  \gamma$ for all $i\neq j$, 
and $\chi_D(D_i,D_j)\leq  \beta$ for $i=j$.
\end{definition}

\begin{lemma}[\cite{FeldmanGRVX17}]\label{lem:sq-lb}
Let $\mathcal{B}(\D,D)$ be a decision problem, where $D$ is 
the reference distribution and $\D$ is a class of distribution. For $\gamma,\beta>0$, 
let $s=\mathrm{SD}(\mathcal{B},\gamma,\beta)$. 
For any $\gamma'>0$, any SQ algorithm for $\mathcal{B}$ requires queries of tolerance 
at most $\sqrt{\gamma+\gamma'}$ or makes at least $s\gamma'/(\beta-\gamma)$ queries.
\end{lemma}

\begin{fact} [Fact 31 from \cite{diakonikolas2022near}] \label{lem:near-orthogonal}
For any constant $0<c<1/2$, there exists a set $V\subseteq \mathbb{S}^{d-1}$ such that $|V|=2^{\Omega_c(d)}$ and for any $u,v\in V$,
$\abs{u\cdot v}\leq c\;.$
\end{fact}


\section{Omitted Proofs from \Cref{sec test to learn}}\label{app test to learn}
In this section, we provide the omitted proofs in \Cref{sec test to learn}.

\subsection{Proof of \Cref{lm error decomposition}}\label{app error decomposition}

We provide the full proof of \Cref{lm error decomposition} here and restate \Cref{lm error decomposition} as \Cref{lm error decomposition re} for convenience.

\begin{lemma}\label{lm error decomposition re}
    Let $(D,f^*,H)$ be any instance of multiclass classification with RCN. Let $h: X \to Y$ be an arbitrary multiclass hypothesis over $X$. Then,
    \begin{align*}
        \err(h) 
        = \sum_{j=1}^k  \Pr(S_j)(1-H_{jj}) + \sum_{i \neq j} \Pr(S_{ji})(H_{jj}-H_{ji}).
    \end{align*}
In particular, if $H_{jj}-H_{ji} \ge 0$ for every $j \in [k], i \neq j$, then 
\begin{align*}
    \opt = \err(f^*) = \sum_{j=1}^k \Pr(S_j)(1-H_{jj}) \;.
\end{align*}
\end{lemma}

\begin{proof}[Proof of \Cref{lm error decomposition}]
   For $j,i \in [k]$, let $x \in A_{ji}$ be any fixed example. Consider two cases, where $j=i$ and $j \neq i$. In the first case, we have 
   \begin{align*}
       \Pr(h(x) \neq y(x)) = \Pr(y(x) \neq j) = 1-H_{jj}.
   \end{align*}
In the latter case, we have 
\begin{align*}
       \Pr(h(x) \neq y(x)) = \Pr(y(x) \neq i) = 1-H_{ji}.
   \end{align*}
Thus, we have 
\begin{align}\label{eq error decomposition}
\begin{split}
    \err(h) & = \sum_{j=1}^k\sum_{i=1}^k \Pr(S_{ji})\Pr(h(x) \neq y(x) \mid x \in S_{ji}) = \sum_{j=1}^k\sum_{i=1}^k \Pr(S_{ji})(1-H_{ji}) \\
    & = \sum_{j=1}^k\sum_{i\neq j} \Pr(S_{ji})(1-H_{ji}) + \left(\Pr(S_{j}) -\sum_{i\neq j} \Pr(S_{ji})\right)(1-H_{jj}) \\
    & = \sum_{j=1}^k \Pr(S_j)(1-H_{jj}) + \sum_{j=1}^k\sum_{i \neq j} \Pr(S_{ji})(H_{jj}-H_{ji}).
    \end{split}
\end{align}
Here, in the second equation, we use the fact that $\Pr(S_{jj}) = \Pr(S_{j}) -\sum_{i\neq j} \Pr(S_{ji}), \forall j \in [k]$.
Since $\Pr(S_{ji}) \ge 0, \forall j,i \in [k]$, we know from \eqref{eq error decomposition} that $\err(h) \ge \sum_{j=1}^k \Pr(S_j)(1-H_{jj})$, if $H_{ii}-H_{ji} \ge 0, \forall i,j \in [k]$. In particular, by \eqref{eq error decomposition}, $\err(f^*) = \sum_{j=1}^k \Pr(S_j)(1-H_{jj})$, which concludes 
\begin{align*}
    \opt = \err(f^*) = \sum_{j=1}^k \Pr(S_j)(1-H_{jj}).
\end{align*}
\end{proof}

\subsection{Proof of \Cref{th test-to-learn}} \label{app test-to-learn lm}

We present the full proof of \Cref{th test-to-learn}. We restate \Cref{th test-to-learn} as \Cref{lm test-to-learn re}

\begin{lemma}
    \label{lm test-to-learn re}
    Let $\D$ be a family of distribution over $X \times Y$ such that each distribution $D \in \D$ corresponds to a multiclass classification instance $(D,f^*,H)$ that satisfies \Cref{def condition}. If there is a Statistical Query (SQ) learning algorithm $\mathcal{A}$ such that for every instance $(D,f^*,H), D \in \D$, $\mathcal{A}$ makes $q$ queries and each of them has tolerance   $\tau$, and outputs a hypothesis $\hat{h}$ such that $\err(\hat{h}) \le \opt+\alpha$, where 
    \begin{align*}
        2\alpha = \sum_{j=1}^{k-1} \Pr(S_j)(H_{jj}-H_{jk}),
    \end{align*}
    then there is an SQ learning algorithm $\mathcal{A}'$ that solves the correlation testing problem defined in \Cref{def correlation testing} by making $q+1$ queries and each of them has error tolearance $\min(\tau, \alpha/2)$.
\end{lemma}

\begin{proof}[Proof of \Cref{th test-to-learn}]
    The algorithm $\mathcal{A'}$ works as follows. 
    We run $\mathcal{A}$ over $D$ to get a hypothesis $\hat{h}: X \to Y$. Given $\hat{h}$, we make one more statistical query $q$ to estimate $\err(\hat{h})$ with tolerance $\alpha/2$. Denote by $\hat{\err}(\hat{h})$ the returned answer of $q$.
    We reject the null hypothesis if $\hat{\err}(\hat{h})<1-H_{kk}-\alpha/2$ and accept the null hypothesis otherwise. The SQ complexity of the algorithm directly follows its definition.
    In the rest of the proof, we prove its correctness.

    If $D=D_0$, since the label $y$ is drawn independently from $x$ and $\Pr(y=k) = H_{kk} \ge \Pr(y=j) = H_{kj}, \forall j \neq k$, any hypothesis $h:X \to Y$ has 
    \begin{align*}
        \err(h) \ge \err(k) = \Pr(y \neq k) = 1-H_{kk}.
    \end{align*}
 This implies that, $\hat{\err}(\hat{h}) \ge 1-H_{kk}-\alpha/2$ and $\mathcal{A'}$ will not reject the null hypothesis.

In the rest of the proof, we will show that if $D \in \D$, the algorithm $\mathcal{A'}$ will reject the null hypothesis. To start with, we will show that $\opt$ is $2\alpha$ far from $1-H_{kk}$. On the one hand, by \Cref{lm error decomposition}, we have 
\begin{align*}
    \err(k) & = \sum_{j=1}^k \Pr(S_j)(1-H_{jj}) + \sum_{j=1}^{k-1} \Pr(S_{jk})(H_{jj}-H_{jk}) \\
    & = \sum_{j=1}^{k-1} \Pr(S_j)(1-H_{jk}) + \Pr(S_k)(1-H_{kk}) \\
    & = \left(1-\Pr(S_k)\right)\sum_{j=1}^{k-1} \frac{\Pr(S_j)}{\left(1-\Pr(S_k)\right)}(1-H_{jk}) + \Pr(S_k)(1-H_{kk}) \\
     & = \left(1-\Pr(S_k)\right)\sum_{j=1}^{k-1}a_j(1-H_{jk}) + \Pr(S_k)(1-H_{kk}) = \sum_{j=1}^k\Pr(S_j)(1-H_{kk}) = 1-H_{kk}.
\end{align*}
Here, the second equation holds because by the definition of the constant hypothesis $S_j = S_{jk}, \forall j \in [k]$ and
the fourth and the fifth equations are followed by \Cref{def condition}. On the other hand, by \Cref{lm error decomposition}, $\opt = \sum_{j=1}^k \Pr(S_j)(1-H_{jj})$. Thus, 
\begin{align*}
  1-H_{kk}- \opt =   \err(k) - \opt = \sum_{j=1}^k \Pr(S_j)(1-H_{jj}) + \sum_{j=1}^{k-1} \Pr(A_{jk})(H_{jj}-H_{jk}) - \sum_{j=1}^k \Pr(S_j)(1-H_{jj}) = 2\alpha,
\end{align*}
which gives us that $\opt=1-H_{kk}-2\alpha$. Given any hypothesis $\hat{h}$ output by $\mathcal{A}$ with $\err(\hat{h}) \le \opt+\alpha$, we have $\err(\hat{h}) \le 1-H_{kk}-\alpha$. Thus, $\hat{\err}(\hat{h}) \ge 1-H_{kk}-2\alpha/2$ and $\mathcal{A'}$ will reject the null hypothesis. 
This concludes the proof of \Cref{th test-to-learn}.

\end{proof}

\section{Omitted Proofs from \Cref{sec hard testing}}\label{app hard testing}

\subsection{Proof of \Cref{lm projection}}\label{app projection}

In this section, we present the proof of \Cref{lm projection}. For convenience, we restate \Cref{lm projection} as \Cref{lm projection re}.

\begin{lemma}[Distribution Projection]\label{lm projection re}
    Let $\D$ be a hidden direction distribution family over $\R^d$ and let $D^{A,a}_v \in \D$ be a distribution that is consistent with an instance of multiclass polynomial classification with RCN $(D^{A,a}_v,f^*,H)$. For every $i \in [k]$, 
    \begin{align*}
    D^{A,a}_v(x \mid y =i) = \sum_{j=1}^{k-1} \frac{a_j H_{ji}}{H_{ki}} P^{A_j}_v(x).    
    \end{align*}
\end{lemma}

\begin{proof}[Proof of \Cref{lm projection}]
We consider the density function of $D^{A,a}_v$ at a fixed point $(x,i)$,    
\begin{align*}
    D^{A,a}_v(x,i) = \sum_{j=1}^{k-1} a_j P^{A_j}_v(x)H_{f^*(x) i} \;.
\end{align*}
By \Cref{def 1D}, we consider two cases for $x$. In the first case, $v\cdot x \in J_\ell$ for some $\ell \in [k-1]$.
By construction of the distribution family $A=(A_1,\dots,A_{k-1})$, $P^{A_j}_v(x)=0, \forall j \neq \ell$. Thus, 
\begin{align*}
    D^{A,a}_v(x,i) = \sum_{j=1}^{k-1} a_j P^{A_j}_v(x)H_{\ell i} = a_\ell P^{A_\ell}_v(x)H_{\ell i}+ \sum_{j \neq \ell} a_j P^{A_j}_v(x)H_{j i} = \sum_{j=1}^{k-1} a_j P^{A_j}_v(x)H_{ji}.
\end{align*}
In the second case, $v\cdot x \not \in I_{in}$ and thus $f^*(x) = k$. In this case, 
\begin{align*}
    D^{A,a}_v(x,i) = \sum_{j=1}^{k-1} a_j P^{A_j}_v(x)H_{k i} = P^{A_1}_v(x)H_{k i} = P^{A_1}_v(x) \sum_{j=1}^{k-1} a_j H_{ji} = \sum_{j=1}^{k-1} a_j P^{A_j}_v(x)H_{ji}.
\end{align*}
Here, the second and the last equation holds because $P^{A_j}_v(x)$ is the same for every $j \in [k]$. The third equation holds because $h_k=\sum_{j=1}^{k-1}h_j$. Since $(D^{A,a}_v,f^*,H)$ satisfies \Cref{def condition}, we know that $\Pr(y=i) = H_{ki}$. Thus, $\forall x \in \R^d$, 
\begin{align*}
    D^{A,a}_v(x \mid y =i) = \sum_{j=1}^{k-1} \frac{a_j H_{ji}}{H_{ki}} P^{A_j}_v(x).
\end{align*}

\end{proof}

\subsection{Proof of \Cref{th hidden family}}\label{app hidden family}

In this section, we present the proof of \Cref{th hidden family}. For convenience, we restate \Cref{th hidden family} as follows.
\begin{theorem}\label{th hidden family re}
    Let $\mathcal{B}(D_0,\D)$ be a correlation testing problem, where $(D_0)_X = N(0,I)$ and $\D$ is a hidden direction distribution family. Suppose there exists some $\nu>0$ such that for $\ell \le t \in \Z_+$, the family of one-dimensional distribution $A_1,\dots,A_{k-1}$ satisfies $ \abs{\E_{x \sim A_i} x^\ell - \gamma_\ell} \le \nu$. Then the following holds: 
    \begin{enumerate} 
        \item \label{prop realize re} Every distribution $D^{A,a}_v \in \D$ is consistent with an instance of multiclass polynomial classification with RCN $(D^{A,a}_v,f^*,H)$ with degree at most $2m$ that satisfies \Cref{def condition}.
        \item \label{prop query re} For any small enough constant $c>0$,
    let $\beta = \max_{i,j}\chi_{N(0,1)}(A_i,A_j)$ and let $\tau:= \nu^2 + c^t\beta$. Any statistical query algorithm that solves $\mathcal{B}(D_0,\D)$ must make a query with accuracy better than $2\sqrt{\tau}$ or make $2^{\Omega_c(d)}\tau/\beta$ queries. 
    \end{enumerate}
\end{theorem}

\begin{proof}[Proof of \Cref{th hidden family}]
We first prove \Cref{prop realize} in \Cref{th hidden family}.

 Since for each $i \in [k-1]$, $J_i$ is a set of $m$ disjoint intervals, we know there is a degree-$2m$ polynomial $p_i(t): \R \to \R$ such that $p_i(t)>0$ if and only if $t \in J_i $. 
  On the other hand, since $I_{in} = \textbf{conv}\bigcup_{j \in [k-1]}J_j$ is a finite interval, there is a degree-2 polynomial $p_k(t): \R \to \R$ such that $p_k(t)>0$ if and only if $t \not\in I_{in}$. Since $J_i \cap J_j = \emptyset, \forall i \neq j$. We know that for each $j \in [k-1]$, if $v\cdot x \in J_j$, then $j=\argmax\{p_1(v\cdot x),\dots, p_{k}(v \cdot x)\}$ and if $v\cdot x \not\in I_{in}, k=\argmax\{p_1(v\cdot x),\dots, p_{k}(v \cdot x)\}$. In particular, $\Pr(v \cdot x \in I_{in} \setminus \bigcup_j J_j)=0$ by the construction of the hidden direction distribution family.
  Thus, $D^{A,a}_v$ is consistent with an instance of multiclass polynomial classification with RCN $(D^{A,a}_v,f^*,H)$, where the marginal distribution is $\sum_{j=1}^{k-1}a_j P_v^{A_j}$ and the ground truth hypothesis $f^*(x) = \argmax\{p_1(v\cdot x),\dots, p_{k}(v \cdot x)\}$. In particular, by the definition of $(D^{A,a}_v,f^*,H)$, it satisfies \Cref{def condition}.

  Next, we prove \Cref{prop query} in \Cref{th hidden family}. Our proof strategy is to make use of \Cref{lem:sq-lb}. To do this, we will bound $\chi_{D_0}(D^{A,a}_v,D^{A,a}_u)$ for $v,u \in S$ for a pair of unit vectors $u,v$.
  For convenience, we mention the following lemma that will be used in the proof.

  \begin{lemma}[Lemma 13 in \cite{diakonikolas2022near}]\label{lm correlation massart}
    Suppose there exists some $\nu>0$ such that for $\ell \le t \in \Z_+$, a univariate distribution $A$ satisfies $ \abs{\E_{x \sim A} x^\ell - \gamma_\ell} \le \nu$, then for every $u,v \in \R^d$, with $\abs{u\cdot v}$ less than a sufficiently small constant, we have 
    \begin{align*}
        \chi_{N(0,I)}(P^A_v,P^A_u) \le \abs{u\cdot v}^t \chi^2(A,N(0,1))+ \nu^2.
    \end{align*}
\end{lemma}

We start by upper-bounding the pairwise correlation $\chi_{D_0}(D^{A,a}_v,D^{A,a}_u)$.

By \Cref{lm projection}, we know that for each $i \in [k]$, 
\[D^i_v = \sum_{j=1}^{k-1}\frac{a_j H_{ji}}{H_{ki}} P^{A_j}_v(x) = P^{\sum_{j=1}^{k-1}\frac{a_j H_{ji}}{H_{ki}} A_j}_v(x) \;.\] 
Since for each $j \in [k-1]$, 
$\abs{\E_{x\sim A_j}x^\ell-\gamma_\ell} \le \nu$ and $\sum_{j=1}^{k-1}\frac{a_j H_{ji}}{H_{ki}}=1$, we know that 
\[\abs{\E_{x\sim \sum_{j=1}^{k-1}\frac{a_j H_{ji}}{H_{ki}} A_j}x^\ell-\gamma_\ell} \le \nu,\]
for $\ell \le t$. 
Thus, we obtain
\begin{align*}
    \chi_{D_0}(D^{A,a}_v,D^{A,a}_u) & = \sum_{i=1}^k H_{ki} \chi_{D_0 \mid y=i} \left(D^{A,a}_v(x \mid y =i),D^{A,a}_u(x \mid y =i)\right) 
     = \sum_{i=1}^k H_{ki} \chi_{N(0,I)} \left(D^i_v,D^i_u\right) \\
    & \le \sum_{i=1}^k H_{ki} \left(\nu^2 + \abs{v\cdot u}^t \chi^2\left(\sum_{j=1}^{k-1}\frac{a_j H_{ji}}{H_{ki}} A_j ,N(0,1)\right) \right) \\
    & = \nu^2 + \abs{v\cdot u}^t \sum_{i=1}^k H_{ki} \chi^2\left(\sum_{j=1}^{k-1}\frac{a_j H_{ji}}{H_{ki}} A_j ,N(0,1)\right)\\
    & \le \nu^2 + \abs{v\cdot u}^t \beta \;.
\end{align*}
Here the first inequality holds because of \Cref{lm correlation massart}, and the last inequality is shown as follows
\begin{align*}
    \chi^2\left(\sum_{j=1}^{k-1}\frac{a_j H_{ji}}{H_{ki}} A_j ,N(0,1)\right) &= \int_{-\infty}^\infty \frac{\sum_{j=1}^{k-1} \frac{a_j H_{ji}}{H_{ki}} A_j(x) \sum_{\ell=1}^{k-1} \frac{a_\ell H_{\ell i}}{H_{ki}} A_\ell(x)}{G(x)} dx -1 \\
    &= \sum_{j=1}^{k-1} \frac{a_j H_{ji}}{H_{ki}} \sum_{\ell=1}^{k-1} \frac{a_\ell H_{\ell i}}{H_{ki}} \left( \int_{-\infty}^\infty \frac{A_j(x) A_\ell(x)}{G(x)} dx -1\right) \\
    & = \sum_{j=1}^{k-1} \frac{a_j H_{ji}}{H_{ki}} \sum_{\ell=1}^{k-1}\frac{a_\ell H_{\ell i}}{H_{ki}} \chi_{N(0,1)} \left(A_i,A_j \right) \le \beta.
\end{align*}
Thus, for every $u,v \in \mathbb{S}^{d-1}$ such that $\abs{u\cdot v} \le c$, we have $\chi_{D_0}(D^{A,a}_v,D^{A,a}_u) \le \nu^2 + c^{-k}\beta=\tau$.

Similarly, we upper bound $\chi^2\left(D^{A,a}_u),N(0,I)\right)$ as follows.

\begin{align*}
    \chi^2\left(D^{A,a}_u,N(0,I)\right) & = \sum_{i=1}^k H_{ki} \chi_{D_0 \mid y=i} \left(D^{A,a}_v(x \mid y =i),D^{A,a}_v(x \mid y =i)\right) 
   \\
   &= \sum_{i=1}^k H_{ki} \chi^2 \left(D^i_v,N(0,I)\right) 
    = \sum_{i=1}^k H_{ki} \chi^2\left(\sum_{j=1}^{k-1}\frac{a_j H_{ji}}{H_{ki}} A_j ,N(0,1)\right) \le \beta \;.
\end{align*}
By \Cref{lem:near-orthogonal}, for any small constant $c>0$, there exists a set $S$ of $2^{\Omega_c(d)}$ unit vectors such that for every $u,v \in S$, $\abs{u\cdot v} \le c$. Thus, $\mathrm{SD}(\mathcal{B},\gamma,\beta) = 2^{\Omega_c(d)}$. By \Cref{lem:sq-lb}, we conclude the proof.

\end{proof}

\subsection{Proof of \Cref{prop distribution}}\label{app proof distribution}

In this section, we present the full proof of \Cref{prop distribution}. For convenience, we restate \Cref{prop distribution} as \Cref{prop distribution re}.

\begin{proposition}\label{prop distribution re}
The univariate distributions $A_1,\dots,A_{k-1}$ constructed in \Cref{def distribution} satisfy
    \begin{enumerate} 
    \item $\exists$ a set of $m$ disjoint intervals $J_i, i \in [k-1]$ such that $A_i(x)>0$, for $x \in J_i$ and $A_i(x)=0$, for $x \in I_{in} \setminus J_i, I_{in} = \textbf{conv}\bigcup_{j \in [k-1]}J_j$.
    \item $\forall x\in \R \setminus I_{in}, A_i(x)=A_j(x), \forall i,j \in [k-1]$.
        \item For $i,j \in [k-1]$, $\chi_{N(0,1)} \left(A_i,A_j\right) \le O(\delta/\xi)^2$.
        \item For $t \in \N$ and for $i \in [k-1]$, $ \abs{\E_{x \sim A_i} x^t - \gamma_t} \le O(t!)\exp(-\Omega(1/\delta^2))+ 4(k-1)\xi(1+2m\delta)^t.$
    \end{enumerate}
\end{proposition}

Before presenting the proof, 
it will be convenient to recall the following property proved by \cite{nasser2022optimal}.
\begin{fact}\label{fact measure}
    For $\delta,\xi>0$, $\abs{\norm{G_{\delta,\xi}}_1-1} \le \exp\left(-\Omega(1/\delta^2)\right), \abs{1/\norm{G_{\delta,\xi}}_1 -1} \le O(1)\exp\left(-\Omega(1/\delta^2)\right)$.
\end{fact}

\begin{proof}[Proof of \Cref{prop distribution}]
    We first prove the first two properties. For $i \in [k-1]$, we define $J_i:= \bigcup_{-m \le n \le m}[n\delta-(4i-3)\xi,n\delta-(4i-5)\xi]$. Notice that $I_{in} = \textbf{conv}\bigcup_{j \in [k-1]}J_j=[-m\delta-(4k-7)\xi,m\delta+\xi]$. By construction, $A_i(x) = A_j(x)$ if $x \not \in I_{in}$. On the other hand, consider $x \in I_{in}$. For $i=1$, and $x \in I_{in}$, $A_1(x)>0$ if and only if $x \in J_1$. By construction for $i \in [k-1]$ and $x \in I_{in}$, $A_i(x)>0$ if and only if $(x+4(i-1))\xi \in J_1$, which is equivalent to $x \in J_i$. Since $\delta>4(k-1)\xi$, we know that $J_i \cap J_j = \emptyset, \forall i \neq j$. This implies that for every $i \in [k-1]$, and $x\in I_{in}$ $A_1(x)>0$ if $x \in J_i$ and $A_i(x)=0$ if $x \in J_j$.

    We next prove the third property. It is convenient to mention the fact that $\chi^2(A_1(x), N(0,1))\le O(\delta/\xi)^2$, proved in \cite{nasser2022optimal}.
    For any pair of $i,j \in [k-1]$, we have 
    \begin{align*}
        \chi_{N(0,1)}\left(A_i,A_j\right) & = \int_{-\infty}^\infty \frac{A_i(x)A_j(x)}{G(x)} dx -1 = \int_{x \not \in I_{in}} \frac{G^2(x)}{G(x)} dx + \int_{x \in I_{in}} \frac{A_i(x)A_j(x)}{G(x)} dx -1 \\
        & \le \chi^2(A_1(x), N(0,1)) + \int_{x  \in I_{in}} \frac{A_i(x)A_j(x)}{G(x)} dx  \\ & = O(\frac{\delta}{\xi})^2 + \int_{x \in I_{in}} \frac{A_i(x)A_j(x)}{G(x)} dx.
    \end{align*}
    Notice that if $i \neq j$, then for each $x \in I_{in}, A_i(x)A_j(x)=0$, which implies that 
    \begin{align*}
        \int_{x \in I_{in}} \frac{A_i(x)A_j(x)}{G(x)} dx =0.
    \end{align*}
    It remains to consider the case where $i=j$. In this case, we have 
    \begin{align*}
        \int_{x \in I_{in}} \frac{A_i^2(x)}{G(x)} dx &\le 2\sum_{0 \le n \le m} \int_{n\delta - (4i-3)\xi}^{n\delta-(4i-5)\xi} \frac{A_i^2(x)}{G(x)} dx   = 2 \left(\frac{\delta}{\xi}\right)^2 \frac{1}{\norm{G_{\delta,\xi}}_1} \sum_{0 \le n \le m} \int_{n\delta - (4i-3)\xi}^{n\delta-(4i-5)\xi} \frac{G^2(x+4(i-1)\xi)}{G(x)} dx \\
        & = 2 \left(\frac{\delta}{\xi}\right)^2 \frac{1}{\norm{G_{\delta,\xi}}_1} \sum_{0 \le n \le m} \int_{n\delta - (4i-3)\xi}^{n\delta-(4i-5)\xi} \frac{1}{\sqrt{2\pi}} \exp\left(-\frac{(x+4(i-1)\xi)^2}{2}\right) \exp((4(i-1)\xi)^2) \\
        & \le 2 \left(\frac{\delta}{\xi}\right)^2 \frac{1}{\norm{G_{\delta,\xi}}_1} \int_{x \in \R} \frac{1}{\sqrt{2\pi}} \exp\left(-\frac{(x+4(i-1)\xi)^2}{2}\right) \exp((4(i-1)\xi)^2) \le O(1) \;.
    \end{align*}
Here, the last inequality holds when $\xi \le 1/k$. Thus, for $i,j \in [k-1]$, $\chi_{N(0,1)} \left(A_i,A_j\right) \le O(\delta/\xi)^2$.

Finally, we prove the last property. It is convenient to mention the fact that $\abs{\E_{x\sim A_1} x^t- \gamma_t} \le O(t!)\exp(-\Omega(1/\delta^2))$, proved in \cite{nasser2022optimal}, which implies that it suffices to upper bound $\abs{\E_{x\sim A_1} x^t- \E_{x \sim A_i} x^t}$. We have 
\begin{align*}
    \abs{\E_{x\sim A_1} x^t- \E_{x \sim A_i} x^t} & = \abs{ \int_{x \in I_{in}} x^t d A_1(x) - \int_{x \in I_{in}} x^t d A_i(x) } = \abs{ \int_{x \in I_{in}} x^t d A_1(x) - \int_{x \in I_{in}} (x-4(i-1)\xi)^t d A_1(x) } \\
    & \le \sup_{x \in I_{in}} \left((x-4(i-1)\xi)^t-x^t\right) \le \sum_{\ell=1}^t \binom{t}{\ell}(4(k-1)\xi)^\ell x^{t-\ell} \le 4(k-1)\xi \sum_{\ell=1}^t \binom{t}{\ell} x^{t-\ell}\\
    &\le 4(k-1)\xi (1+\abs{x})^t \le 4(k-1)\xi (1+2m\delta)^t.
\end{align*}
    This concludes the proof of \Cref{prop distribution}.
\end{proof}

\section{Omitted Proofs from \Cref{sec hard main}} \label{apd sec hard main}
In this section, we provide the omitted proofs in \Cref{sec hard main}.

\subsection{Proof of \Cref{th additive}} \label{app additive}
We present the full proof of \Cref{th additive}.
For convenience, we restate \Cref{th additive} below.
\begin{theorem}\label{th additive re}
There is a matrix $H \in [0,1]^{3\times3}$ with $H_{ii}-H_{ij} \ge 0.1, \forall i \neq j \in [3]$ such that any SQ algorithm $\mathcal{A}$ that learns multiclass linear classifiers with random classification noise specified by $H$ on $\R^d$ to error $\opt+\epsilon, \epsilon \in (0,1)$ either 
\begin{enumerate} 
    \item [(a)] requires at least $d^{\Omega(\log^{1.98}(d)/\eps^{1.98})}$ queries, or
    \item [(b)] requires a query of accuracy at least $d^{-\Omega(\log^{1.98}(d)/\eps^{1.98})}$. 
\end{enumerate}
\end{theorem}
\begin{proof}
Consider following noise matrix, 
\begin{align*}
    H = \begin{pmatrix}
        0.6 & 0 & 0.4 \\
        0 & 0.6 & 0.4 \\
        0.3 & 0.3 & 0.4
    \end{pmatrix}.
\end{align*}
Notice that one can
reduce learning polynomial classifiers with RCN to MLC with RCN using the Veronese mapping.
Suppose we have an algorithm $\A$ for solving multiclass linear classification problems. Then given an input distribution $D$ of $(x,y)$ over $\R^N\times [k]$ consistent with an instance of  
multiclass degree-$m$ polynomial classification problem. 
We apply the Veronese mapping $V(x)\eqdef (x,1)^{\otimes m}$ on $x$. The distribution of $(V(x),y)$ over $\R^{(N+1)^{O(m)}}\times [k]$ is consistent with an instance of MLC with RCN specified by $H$.
Therefore, to get an SQ lower bound for MLC over $\R^d$, it suffices for us to give an SQ lower bound for learning degree-$m$ polynomial classifiers over $\R^N$, where $d=N^{O(m)}$.

Furthermore, by \Cref{th test-to-learn}, to get the SQ lower bound for learning polynomial classifiers, it suffices for us to give an SQ lower bound on a corresponding testing problem.
Therefore, we construct a distribution family $\D$ of joint distributions of $(x,y)$ over $\R^N\times [k]$ such that 
each distribution in $\D$ is consistent with a multiclass polynomial classification problem with RCN $(D,f^*,H)$ as required by \Cref{th test-to-learn}.

The construct $\D$ as the hidden direction distribution family defined in \Cref{def 1D}.
We choose $A_1$ and $A_2$ as specified in \Cref{def distribution}, $a=(1/2,1/2)$. Since the noise matrix $H$ satisfies $h_3 = (h_1+h_2)/2$, by \Cref{th hidden family}, we know that each distribution $D$ is consistent with an instance of multiclass polynomial classification problem with degree-$O(m)$ with RCN specified by $H$. To make use of \Cref{th hidden family} to get an SQ lower bound, it remains to choose parameters for $A_1,A_2$ such that it is hard to solve $\mathcal{B}(D_0,\D)$. 

Fix any small enough constant $\epsilon>0$.
We choose the parameters $m,\delta$ such that $m\delta = \eps$.
By \Cref{prop distribution}, we know that for every $t \in \N$, we have 
\begin{align*}
    \abs{\E_{x \sim A_i} x^t - \gamma_t} & \le O(t!)\exp\left(-\Omega(1/\delta^2)\right)+ 12\xi(1+2m\delta)^t \\
    & \le O(1) \left( \exp\left(t\log(t)-\Omega(1/\delta^2)\right)+ \xi \exp(2\eps t) \right).
\end{align*}
We choose $\delta=1/\sqrt{N}, m = \eps \sqrt{N}$ and $\xi=\exp(-2N^{0.99})$. For any $t\le N^{0.99}$, we have 
\begin{align*}
    \abs{\E_{x \sim A_i} x^\ell - \gamma_\ell}& \le O\left( \exp\left(t\log(t)-\Omega(1/\delta^2)\right)+ \xi \exp(t) \right) \\
    &\le \exp(-\Omega(N))+\exp(-2N^{0.99}+N^{0.99})=\exp(-\Omega(N^{0.99})) =: \nu.
\end{align*}


By \Cref{th test-to-learn}, we know that any statistical query learning algorithm that learns $(D^{A,a}_v,f^*,H)$ up to error $\opt+\alpha$, where $\alpha = \frac{1}{2}\sum_{j=1}^2(H_{jj}-H_{j3})\Pr(S_j)$ can solve $\mathcal{B}(D_0,\D)$. 
By the construction of $D^{A,a}_v$ and $H$, 
\begin{align*}
\alpha & = 0.1 \Pr_{x\sim A_1}(x \in I_{in}) = 0.1 \sum_{-m \le n \le m} \int_{n\delta-\xi}^{n\delta+\xi}\frac{\delta}{\xi} \frac{1}{\norm{G_{\delta,\xi}}} G(x) dx \\
& \ge \Omega(1) \sum_{-m \le n \le m} \int_{n\delta-\xi}^{n\delta+\xi} (\frac{\delta}{\xi})G(x) dx 
 \ge \Omega(1) \sum_{-m \le n \le m} 2\xi (\frac{\delta}{\xi}) G(2m\delta) \ge \Omega(1)(2m+1)\delta = \Omega(\epsilon).
\end{align*}
Here, the first inequality holds because of \Cref{fact measure}, the second inequality holds because $G(x)$ is decreasing with respect to $\abs{x}$ and the last inequality holds because $\epsilon=m\delta$. This implies that any statistical learning algorithm that learns the multiclass polynomial classification problem $(D^{A,a}_v,f^*,H), v \in \mathbb{S}^{N-1}$ up to error $\opt+O(\epsilon)$ must make at least $2^{\Omega(N)}$ statistical queries or a query with accuracy better than $\exp(-\Omega(N^{0.99}))$. 

Finally, we conclude the proof of \Cref{th additive} by embedding the multiclass polynomial classification problem $(D^{A,a}_v,f^*,H), v \in \mathbb{S}^{N-1}$ into $\R^d, d= O(N^m)$ as a multiclass linear classification problem via choosing $m$ properly. By choosing $m= \eps \sqrt{N}$, we obtain that 
\begin{align*}
    \log(d) = \Theta(m \log(N)) = \Theta(\eps \sqrt{N} \log(N)). 
\end{align*}

any statistical learning algorithm that learns the multiclass linear classification problem over $\R^d$ up to error $\opt+O(\epsilon)$ must make at least 
\begin{align*}
 \exp(\Omega(N^{0.99})) = d^{\Omega(N^{0.99}/\log(d))} = d^{\Tilde{\Omega}(\log^{1.98}(d)/\eps^{1.98})}   
\end{align*}
statistical queries or a query with accuracy better than $\exp(-\Omega(N^{0.99}))=d^{-\Tilde{\Omega}(\log^{1.98}(d)/\eps^{1.98})}$.

\end{proof}

\subsection{Proof of \Cref{th multiply}}\label{app multiply}

We give the proof of \Cref{th multiply} below. For convenience, we state \Cref{th multiply} as \Cref{th multiply re}.
\begin{theorem}\label{th multiply re}
For any $k\in \Z_+$ and $k\geq 3$,
there is a noise matrix $H\in [0,1]^{k\times k}$ such that $\max_{i,j} H_{i,i}-H_{i,j}=\zeta>0$ and has the following property:
For any sufficiently large $d\in \Z_+$,
any SQ algorithm $A$ that distribution-free learns multiclass linear classifiers with random classification noise specified by $H$ on $\R^d$ to error $1-1/k-\zeta-2\mu$ either 
\begin{enumerate} 
    \item [(a)] requires at least $q$ queries, or
    \item [(b)] requires a query of tolerance at most $\mu$,
\end{enumerate}
where $\min(q,1/\mu^2)=d^{\Omega(\log^{0.99} d)}$.
In particular, this holds even if $\opt \le 1/k+\zeta+1/k^3$.
\end{theorem}

\begin{proof}[Proof of \Cref{th multiply}]
    Given $k\in \Z_+$ and $k\geq 3$, we construct the corresponding noise matrix $H$ as
    \begin{align*}
    H = \begin{pmatrix}
        (k-1)/k-\zeta & 0 & \cdots & 1/k+\zeta \\
        0 & (k-1)/k-\zeta & \cdots & 1/k+\zeta \\
        \cdots & \cdots  & \cdots & \cdots\\
        1/k-\zeta/k & 1/k-\zeta/k & \cdots & 1/k+\zeta
    \end{pmatrix}.
    \end{align*}
    Namely, for all $i\in [k-1]$, the $i$th row is defined as 
    $H_{i,i}=(k-1)/k-\zeta$, $H_{i,k}=1/k+\zeta$ and $H_{i,j}=0$ for any $j$ such that $j\neq i$ and $j\neq k$.
    Then the $k$th row is defined as $H_{k,k}=1/k+\zeta$ and $H_{k,j}=1/k-\zeta/k$ for any $j\neq k$.
    
    Given the noise matrix $H$, we will construct a hidden direction distribution family $\D$ over $\R^N$ that is consistent with a family of multiclass polynomial classification problems with RCN $(D^{A,a},f^*,H)$ using polynomials of degree $O(m)$ for some $m$ to be determined later and prove the SQ hardness for these multiclass polynomial classification problems. Given the hardness, the hard instance of multiclass linear classification problems with RCN, would be an instance in $\R^d$, where $d=N^{O(m)}$ of the form $(M(x),y), (x,y) \sim D^{A,a}$, where $M(x):\R^N \to \R^d$ defined as $M(x)=[x,1]^{\otimes m}$ is the degree-$m$ Veronese mapping that maps a vector $x \in \R^N$ to monomials of degree at most $m$. 

    To start with, we construct the hidden direction distribution family $\D$ using the distributions $A_1,\dots,A_{k-1}$ constructed in \Cref{def distribution}. Notice that $H$ satisfies \Cref{def condition} because
    $h_k=\sum_{i\in [k-1]}a_i h_i$, where $a_i=1/(k-1)$ for all $i$. We know from \Cref{th hidden family} that for every $A=(A_1,\dots,A_{k-1})$, where $A_1,\dots,A_{k-1}$ are one dimension distributions constructed in \Cref{def distribution}, $\D = \{D^{A,a}_v\}_{v \in \mathbb{S}^{N-1}}$ is a hidden direction distribution family consistent with a family of multiclass polynomial classification problems with RCN $(D^{A,a},f^*,H)$ using polynomials of degree $O(m)$. Now, we choose parameters for $A_1,\dots,A_{k-1}$ such that it is hard to solve $\mathcal{B}(D_0,\D)$. 
    
    We will choose $\delta=1/\sqrt{N}$, $\xi=\exp(-N^{0.99}\log k)$ and $m=\lceil C\sqrt{\log k}/\delta\rceil$, where $C$ is a sufficiently large constant.
    By \Cref{prop distribution}, we know that for every $t \in \N$, we have 
    \begin{align*}
    \abs{\E_{x \sim A_i} x^t - \gamma_t} & \le O(t!)\exp\left(-\Omega(1/\delta^2)\right)+ 12\xi(1+2m\delta)^t \\
    & \le O\left( \exp\left(t\log(t)-\Omega(1/\delta^2)\right)+ \xi \exp(t\sqrt{\log k}) \right).
    \end{align*}
Therefore, we get for any $t\leq N^{0.99}$,
\begin{align*}
    \abs{\E_{x \sim A_i} x^\ell - \gamma_\ell}& \le O\left( \exp\left(t\log(t)-\Omega(1/\delta^2)\right)+ \xi \exp(t\sqrt{\log k}) \right) \\
    &\le \exp(-\Omega(N))+\exp(-N^{0.99}\log k+N^{0.98}\sqrt{\log k})=\exp(-\Omega(N^{0.99})) =: \nu.
\end{align*}
By \Cref{prop distribution}, we know that 
\begin{align*}
    \beta := \max_{i,j} \chi_{N(0,1)}(A_i,A_j) = O(\delta/\xi)^2 = \exp(O(N^{0.99}\log k)).
\end{align*}
This implies that (we take the constant $c$ in \Cref{lm correlation massart} as $c^{-1}>2^{10\log k}$)
\begin{align*}
    \tau:= \nu^2 +c^t\beta \le \exp(-\Omega(N^{0.99}))+\exp(-\Omega(N^{0.99}\log k))=\exp(-\Omega(N^{0.99})).
\end{align*}
By \Cref{th hidden family}, we know that to solve the correlation testing problem $\mathcal{B}(D_0,\D)$, one need at least $2^{\Omega_k(N)}\tau/\beta = 2^{\Omega_k(N)}$ (given $N$ is at least a sufficiently large constant depending on $N$) statistical queries or a query with accuracy better than $2\sqrt{\tau}=\exp(-\Omega(N^{0.99}))$. 
Furthermore, since $N$ is at least a sufficiently large constant depending on $k$, the lower bound on the number of queries is $2^{\Omega_k(N)}\geq 2^{\Omega(N^{0.99})}$, where we simply take $N^{0.01}\geq c(k)$ and $c(k)$ is the constant factor in $2^{\Omega_k(N)}$ that depends on $k$.

Notice that by \Cref{lm error decomposition}, for any $D\in \cal D$, we have
\begin{align*}
    \opt= & (1/k+\zeta)\pr_{t\sim G_{\delta,\xi}} [t\in [-m\delta,m\delta]]+((k-1)/k-\zeta)\pr_{t\sim G_{\delta,\xi}} [t\in (-\infty,-m\delta-\delta/2]\cup [m\delta+\delta/2,\infty)]\\
    \le & (1/k+\zeta)+\pr_{t\sim G_{\delta,\xi}} [t\in (-\infty,-m\delta-\delta/2]\cup [m\delta+\delta/2,\infty)]\\
    \le & (1/k+\zeta)+2\sum_{i>m}\int_{i\delta-\xi}^{i\delta+\xi} G_{\delta,\xi}(t)dt \\
    \le & (1/k+\zeta)+2\pr_{t\sim \normal(0,1)}[t\geq m\delta]
    =(1/k+\zeta)+1/\poly(k)\; .
\end{align*}
On the other hand, if the input distribution is $D_0$, then every hypothesis $h$ has an error 
\[\err_{D_0}(h) \ge  1-(1/k+\zeta)=1-1/k-\zeta.\]
Therefore, any algorithm for learning multiclass polynomial classification with RCN matrix $H$ and achieving an error better than $1-1/k-\zeta-2\mu$ can be used to solve $\mathcal{B}(D_0,\D)$ with one more query with accuracy $2\mu$.  Thus, any such algorithm must either uses $q$ queries or a query of tolerance at most $\mu$, where $q=1/\mu=2^{\Omega(N^{0.99})}$.

Finally, we conclude the proof of \Cref{th additive} by embedding the multiclass polynomial classification problem $(D^{A,a}_v,f^*,H), v \in \mathbb{S}^{N-1}$ into $\R^d, d= O(N^m)$ as a multiclass linear classification problem and rewrite. 
Given $d=O(N^m)$, we get that $2^{\Omega(N^{0.99})}=d^{\Omega(N^{0.99}/m)}=d^{\Omega(N^{0.49}/(c\sqrt{\log k}))}\geq d^{\Omega(N^{0.45})}\geq d^{\Omega(m^{0.9})}\geq d^{\Omega((\log d)^{0.9})}$.
Therefore, 
any SQ learning algorithm that learns the multiclass linear classification problem over $\R^d$ to error  
$1-(1/k+\zeta)=1-1/k-\zeta$ (even given $\opt\leq 1/k+\zeta+1/\poly(k)$) must make at least $d^{\Omega((\log d)^{0.9})}$ statistical queries or a query with accuracy better than $d^{-\Omega((\log d)^{0.9})}$. 
\end{proof}

\subsection{Proof of \Cref{cor approximation} and \Cref{cor beat constant}} \label{app corollary}
We present the proof of \Cref{cor approximation} and \Cref{cor beat constant}. For convenience, we restate \Cref{cor approximation} and \Cref{cor beat constant} as \Cref{cor approximation re} and \Cref{cor beat constant re} respectively.
\begin{corollary} [SQ hardness of approximate learning] \label{cor approximation re}
    For any $C>1$, there exists a noise matrix $H\in [0,1]^{k\times k}$, where $k=O(C)$ and $\min_{i,j} H_{i,i}-H_{i,j}=\Omega(1/C)$ that has the following property: For any $d\in \Z_+$ that is at least a sufficiently large constant depending on $\alpha$, any SQ algorithm $A$ that distribution-free learns multiclass linear classifiers on $\R^d$ with RCN parameterized by $H$ to error $C\opt$ given $\opt=\Omega(1/C)$ either 
    \begin{enumerate}
        \item [(a)] requires at least $d^{\Omega(\log^{0.99} d)}$ queries, or
        \item [(b)] requires a query of tolerance at most $1/d^{\Omega(\log^{0.99} d)}$.
    \end{enumerate}
\end{corollary}
\begin{proof}
    This directly follows from \Cref{th multiply}, where we take $k=\lceil 3C\rceil$ and $\zeta=1/(100k)$.
    Then we have $\opt=1/k+\zeta+1/k^3=1.01/k+1/k^3$.
    An algorithm that achieves error $\alpha\opt$ given $\opt=O(1/C)$ will in this case, output a hypothesis with error $\alpha\opt\le (k/3)\opt\le 2/3$.
    Notice that the SQ lower bound in \Cref{th multiply} holds against any algorithm that outputs a hypothesis with error at most $1-1/k-\zeta-1/\poly(d)=1-1/k-0.01/k-1/\poly(d)\ge 2/3$.
    This completes the proof.
\end{proof}

\begin{corollary} [SQ hardness of beating random guess] \label{cor beat constant re}
    For any $k\in \Z_+$ and $k\geq 3$, there is a noise matrix $H\in [0,1]^{k\times k}$ that $\min_{i,j} H_{i,i}-H_{i,j}=O(1/d)$ and has the following property:
    For any $d\in \Z_+$ that is at least a sufficiently large constant depending on $k$, any SQ algorithm $A$ that distribution-free learns multiclass linear classifiers on $\R^d$ with RCN parameterized by $H$ to error $1-1/k-1/\poly(d)$ given $\opt=O(1/k)$ either 
    \begin{enumerate}
        \item [(a)] requires at least $d^{\Omega(\log^{0.99} d)}$ queries, or
        \item [(b)] requires a query of tolerance at most $1/d^{\Omega(\log^{0.99} d)}$,
    \end{enumerate}
\end{corollary}
\begin{proof}
    This directly follows from \Cref{th multiply}.
    Suppose that there is an algorithm achieving error $1-1/k-1/d^c$ for any $c>0$.
    Then we take $\zeta=1/(2d^c)$.
    Given $d$ is a sufficiently large constant depending on $k$, it is easy to check that $\opt=1/k+\zeta+1/k^3=1/k+1/(2d^c)+1/k^3=O(1/k)$.
    Furthermore, the SQ lower bound holds against any algorithm that outputs a hypothesis with error at most $1-1/k-\zeta-1/\poly(d)=1-1/k-\zeta-2\tau\ge 1-1/k-1/d^c$,
    where the last inequality follows from $\tau=d^{\Omega(\log d)^{0.99}}$.
    This completes the proof.
\end{proof}

\begin{remark}
   {\em  We want to remark that in \Cref{cor beat constant}, to rule out an efficient learning algorithm that has a better error guarantee than $1-1/k$, any choice of $\zeta=o_{k}(1)$ is sufficient, and the separation $\min_{i\neq j} H_{ii}-H_{ij}$ of $H$ is in fact $O(\zeta)$.}
\end{remark}

\end{document}


\end{document}